%% file: main.tex
\documentclass{article}

\PassOptionsToPackage{numbers, compress}{natbib}
\usepackage[preprint]{neurips_2026}

\input{packages}

\input{notations}

\title{
    \algo: LLM Safety via Learning \\Barrier Steering
}

\author{
    Thanh Q. Tran$^{1}$\thanks{Equal contributions and $^\dagger$corresponding authors.}\hspace{1.6mm} 
    \quad Arun Verma$^{2\hspace{0.3mm}*\dagger}$
    \quad Kiwan Wong$^{3}$ 
    \quad Bryan Kian Hsiang Low$^{1,3}$ \\
    \textbf{Daniela Rus}$^{3,2}$
    \quad \textbf{Wei Xiao}$^{4,3\hspace{0.3mm}\dagger}$
    \\
    $^{1}$Department of Computer Science, National University of Singapore, Republic of Singapore\\
    $^{2}$Singapore-MIT Alliance for Research and Technology Centre, Republic of Singapore \\
    $^{3}$CSAIL, Massachusetts Institute of Technology, USA\\
    $^{4}$Worcester Polytechnic Institute, Worcester, MA, USA \\
    \texttt{thanhtran@u.nus.edu}
    \quad \texttt{arun.verma@smart.mit.edu}
    \quad \texttt{kiwan588@mit.edu} \\
    \texttt{lowkh@comp.nus.edu.sg}
    \quad \texttt{rus@csail.mit.edu} 
    \quad \texttt{wxiao3@wpi.edu}
}

\begin{document}

    \maketitle

    \begin{abstract}
        Despite the strong performance of large language models (LLMs) across diverse tasks,
        their susceptibility to adversarial attacks and unsafe content generation remains a
        significant obstacle to deployment, particularly in high-stakes settings.
        Addressing this challenge requires safety mechanisms that are both practically effective
        and theoretically grounded.
        In this paper, we introduce \algo{}, a novel inference-time framework that improves response safety
        by embedding learned nonlinear safety constraints directly into the model's latent
        representation space.
        \algo{} treats hidden-state safety classifiers as Control Barrier Functions (CBFs),
        enabling constraint-guided steering of unsafe latent trajectories during generation.
        By composing multiple safety constraints through efficient constraint merging without
        modifying the underlying LLM parameters, \algo{} preserves model utility.
        We provide theoretical results showing that applying CBFs in latent space yields a
        principled, modular, and computationally efficient approach for steering with respect to learned
        safety constraints, with guarantees conditional on the learned barriers capturing the
        intended safety property.
        Our extensive experimental results across multiple model families and datasets demonstrate that \algo{}
        substantially reduces adversarial attack success rates and unsafe generations,
        outperforming existing methods.
        The code is available in our \href{https://github.com/thanhquangtran/BarrierSteer}{GitHub repository}.
    \end{abstract}



    \section{Introduction}
    \label{sec:introduction}
    \input{latex/introduction}

    \section{Related Work}
    \label{sec:related_work}
    \input{latex/related_work}

    \section{Background and Motivation}
    \label{sec:problem}
    \input{latex/problem}

    \section{\algo{}: LLM Steering with Learned Control Barrier Functions}
    \label{sec:algorithm}
    \input{latex/algorithm}

    \section{Experiments}
    \label{sec:experiments}
    \input{latex/experiment}

    \section{Conclusion}
    \label{sec:conclusion}
    \input{latex/conclusion}

	\bibliographystyle{unsrt}
	\bibliography{references}

    \clearpage
    \appendix
    \begin{center}
    {\bf \Large{Appendix of ``\algo: LLM Safety via Learning Barrier Steering''}}
    \end{center}
    \DoToC
    \clearpage

    \input{latex/appendix}
    \FloatBarrier

    \input{latex/additional_related_work}
    \FloatBarrier

    \input{latex/faq}
    \FloatBarrier

    \input{latex/limitations}
    \FloatBarrier

    \input{latex/broader_impact}
    \FloatBarrier

    \input{latex/llm_usage}
    \FloatBarrier

    \hrule height 0.5mm

\end{document}

%% file: packages.tex
\usepackage[utf8]{inputenc}   
\usepackage[T1]{fontenc}    	
\usepackage{url}              		 
\usepackage{booktabs}       	
\usepackage{amsfonts}       	
\usepackage{nicefrac}       	  
\usepackage{microtype}      	

\usepackage{graphicx}
\usepackage{color}
\usepackage{tikz}
\usepackage{rotating}
\usepackage{tabularx}
\usepackage{pdflscape}
\usepackage{amsmath,amsthm,amssymb}
\usepackage{algorithm,algorithmic}
\usepackage{bbm,dsfont}
\usepackage{subfig}
\usepackage{caption}
\usepackage{wrapfig}
\usepackage{placeins}
\usepackage{cancel}
\usepackage{enumerate,cases}
\usepackage{thmtools,thm-restate}
\usepackage{mathtools}

\usepackage{multirow}
\usepackage{makecell}
\usepackage{setspace}
\usepackage{pifont}
\usepackage{array}
\usepackage{enumitem}
\usepackage{lineno}     
\usepackage{titletoc}

\definecolor{stagepurple}{RGB}{118,70,180}
\definecolor{stagegreen}{RGB}{46,139,87}
\DeclareRobustCommand{\stageicon}[2]{%
  \tikz[baseline=(n.base)]{\node[inner sep=0pt, circle, fill=#1, text=white, minimum size=1.05em, font=\tiny\bfseries] (n) {#2};}%
}

\allowdisplaybreaks

\usepackage{hyperref}		 
\hypersetup{
	colorlinks	= true,		 
	urlcolor     = blue,	 
	linkcolor	 = purple, 	 
	citecolor    = violet  
}

\newcommand\DoToC{%
  \startcontents
  \printcontents{}{1}{\textbf{Table of Contents}\vskip3pt\hrule\vskip5pt}
  \vskip3pt\hrule\vskip5pt
}

\usepackage[capitalize]{cleveref}
\crefformat{figure}{Fig.~#2#1#3}       
\crefformat{table}{Tab.~#2#1#3}        
\crefformat{equation}{Eq.~(#2#1#3)}    
\crefformat{section}{Sec.~#2#1#3}      
\crefformat{appendix}{App.~#2#1#3}     
\crefformat{algorithm}{Alg.~#2#1#3}    

\usepackage{todonotes}

\usepackage{newtxtext}
\usepackage{silence}
\theoremstyle{remark}
\newtheorem*{remark}{Remark}

%% file: notations.tex
\newcommand{\safe}{\textnormal{\text{safe}}}
\newcommand{\unsafe}{\textnormal{\text{unsafe}}}
\newcommand{\algo}{\textsc{BarrierSteer}}

\newcommand{\para}[1]{\textbf{#1}~}

\renewcommand{\phi}{\varphi}
\renewcommand{\epsilon}{\varepsilon}

\newcommand{\norm}[1]{\left\|#1\right\|}

\newcommand{\el}{\end{flushleft}}
\newcommand{\bl}{\begin{flushleft}}

\newcommand{\squishlisttwo}{
 \begin{list}{$\bullet$}
  { \setlength{\itemsep}{1pt}
     \setlength{\parsep}{0pt}
    \setlength{\topsep}{0pt}
    \setlength{\partopsep}{0pt}
    \setlength{\leftmargin}{1em}
    \setlength{\labelwidth}{1.5em}
    \setlength{\labelsep}{0.5em} } 
}
\newcommand{\squishend}{\end{list}}

\newcommand{\cC}{\mathcal{C}}

\theoremstyle{plain}

\newtheorem{defi}{Definition}

%% file: latex/introduction.tex

\begin{wrapfigure}[10]{r}{0.51\textwidth}
    \vspace{-16mm}
    \centering
    \includegraphics[width=0.98\linewidth]{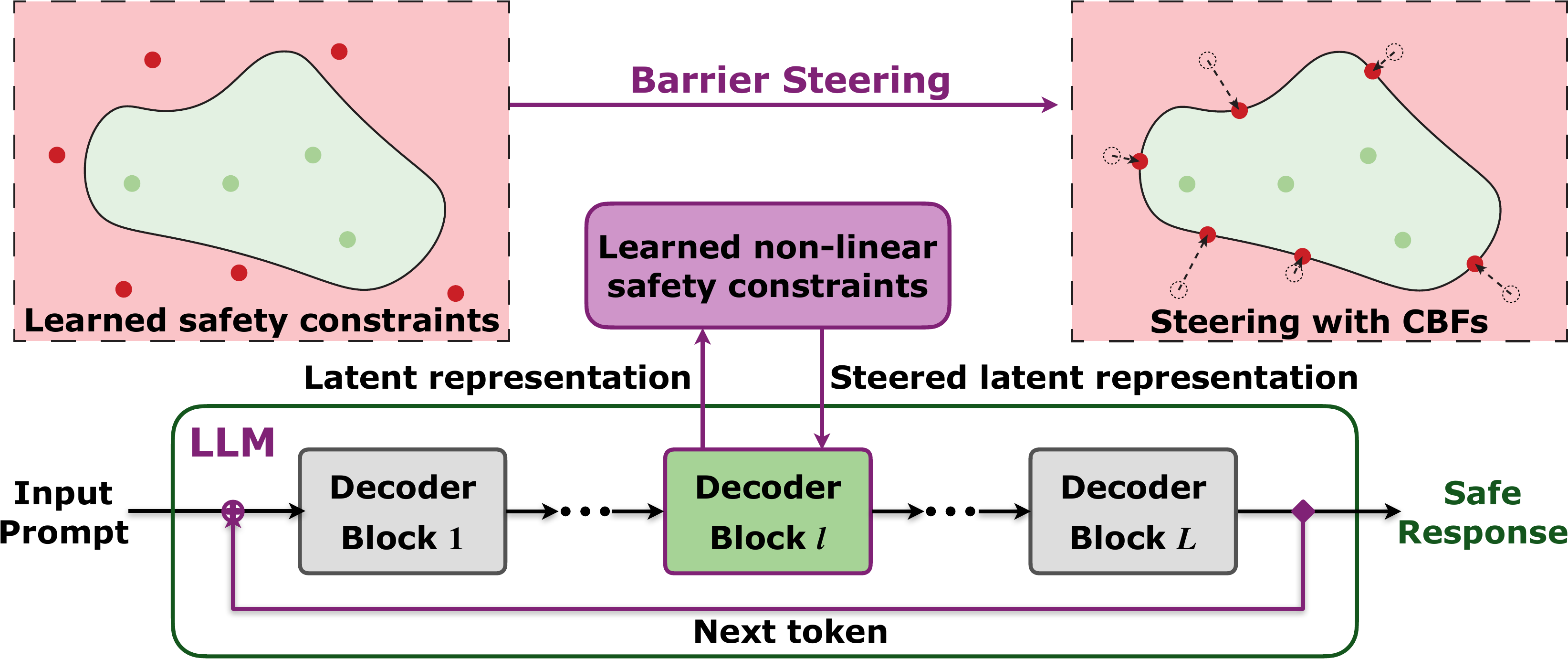}
    \vspace{-1mm}
    \caption{
        \textbf{\algo{} for LLM Safety.} 
        Unlike steering with fixed directions, \algo{} derives adaptive corrections at each token, steering hidden states towards the learned safety boundaries using Control Barrier Functions (CBFs). 
    }
    \label{fig:BarrierSteer}
\end{wrapfigure}
As large language models (LLMs)~\citep{claude2,ArXiv23_google2023palm,ArXiv23_openai2023gpt,ArXiv23_touvron2023llama} become increasingly widespread due to their strong capabilities~\citep{chen2021evaluating,wei2022chain,taori2023alpaca,ArXiv23_zhao2023survey}, their safety risks have become a central concern~\citep{anwar2024foundational,dalrymple2024towards}.
Despite alignment efforts~\citep{NeurIPS22_ouyang2022training,arXiv22_bai2022training}, LLMs can still be manipulated by adversarial attacks~\citep{zou2023universal,ACL24_cao2024defending} to generate harmful content~\citep{EMNLP2020_gehman2020realtoxicityprompts,chua2024ai,zhang2025guardians}.
Such failures can have serious real-world consequences, creating significant legal and ethical risks when deploying LLMs in high-stakes applications~\citep{hung2023walking,fareed2025systematic}.

Existing LLM safety approaches can be broadly grouped into training-time and inference-time defenses. 
Training-time defenses~\citep{NeurIPS22_ouyang2022training,arXiv22_bai2022training,bai2022constitutionalaiharmlessnessai} update model parameters to improve safety, but aligned models can remain vulnerable to adversarial prompts~\citep{zou2023universal,liu2023autodan,shen2024anything,zeng2024johnny}.
Inference-time methods therefore provide a complementary layer of defense, enabling safety interventions without modifying the underlying model parameters~\citep{turner2023steering,sharma2025constitutionalclassifiersdefendinguniversal,cunningham2026constitutional}.

Recent inference-time defenses improve LLM safety but often introduce significant trade-offs.
One line of work detects unsafe outputs from generated text.
Post-hoc filters, such as heuristic keyword-based filters,
are brittle to paraphrases~\citep{wei2023jailbroken,kandpal2023backdoor}.
Language-model classifiers provide stronger detection~\citep{sharma2025constitutionalclassifiersdefendinguniversal, inan2023llamaguardllmbasedinputoutput,han2024wildguardopenonestopmoderation}, but require running an additional model, which can add substantial overhead during deployment.
To reduce this cost, recent work studies classifiers that detect unsafe behavior from intermediate hidden states rather than generated text, enabling low-latency safety detection~\citep{cunningham2026constitutional}.
However, detection identifies unsafe states but not how to steer them toward safer continuations.

A recent adaptive steering approach, \texttt{SaP}, treats classifiers as latent-space
constraints~\citep{ICML2025_chenlearning}.
While promising, \texttt{SaP} is limited to linear classifiers and requires expensive optimization-based steering.
We propose \algo{} to address these limitations by treating learned hidden-state safety classifiers as Control Barrier Functions (CBFs), a standard control-theoretic tool for maintaining dynamical states within safe regions~\citep{Ames2017,Xiao2021TAC2}.
This enables nonlinear safety boundaries, closed-form steering, and principled composition of multiple constraints.
Our key contributions are summarized as follows:

\vspace{-3mm}
\begin{itemize}[leftmargin=*]
    \setlength\itemsep{0.0em}
    \setlength{\itemindent}{2pt}

    \item \textbf{Nonlinear latent safety constraints:}
    We introduce \algo{}, which learns nonlinear hidden-state safety classifiers and treats them as Control Barrier Functions (CBFs), providing a principled bridge between latent safety detection and constraint-based steering.

    \item \textbf{Closed-form CBF steering and composition:}
    We derive efficient closed-form steering updates for both single and composed CBF constraints, including a Log-Sum-Exp composition rule that enables jointly enforcing multiple safety constraints during autoregressive decoding. 

    \item \textbf{Effective low-latency jailbreak defense:}
    We demonstrate empirically across four model families and ten adversarial attacks that \algo{} reduces attack success rates substantially while preserving utility (MMLU, GSM8K) and avoiding the over-refusal observed in fixed-direction steering baselines, with steering latency~58–79$\times$ faster than the closest baseline (SaP).
    
\end{itemize}

%% file: latex/related_work.tex

\para{LLM Safety Techniques.}
Improving the safety of LLMs has become a central challenge in modern AI research~\citep{anwar2024foundational,dalrymple2024towards}.
A major line of work addresses this challenge through training-time alignment, where model parameters are updated using safety, preference, or refusal data.
Representative methods include RLHF~\citep{NeurIPS22_ouyang2022training,arXiv22_bai2022training,bai2022constitutionalaiharmlessnessai}, \texttt{Safe-RLHF}~\citep{dai2023safe}, \texttt{DPO}~\citep{arXiv23_rafailov2023direct,zhou2024beyond,liu2024enhancing}, adversarial training~\citep{ganguli2022redteaminglanguagemodels}, and \texttt{Circuit Breakers}~\citep{zou2024improving}, which incorporate safety through preference optimization, constrained reward modeling, adversarial data, or representation engineering.
However, aligned models can remain vulnerable to adversarial or jailbreak prompts~\citep{zou2023universal,liu2023autodan,shen2024anything,zeng2024johnny}.
This motivates inference-time safety mechanisms that can intervene during deployment without changing the base model parameters.

\para{Inference-time Safety.}
One line of inference-time safety work detects unsafe outputs after generation.
Simple post-hoc filters, such as keyword- or rule-based filters, are brittle to paraphrases and obfuscation~\citep{wei2023jailbroken,kandpal2023backdoor}.
Language-model classifiers provide stronger detection~\citep{sharma2025constitutionalclassifiersdefendinguniversal,inan2023llamaguardllmbasedinputoutput,han2024wildguardopenonestopmoderation}, but require running an additional model, which can add substantial overhead during deployment.
To reduce this cost, recent work studies classifiers that detect unsafe behavior from intermediate hidden states rather than generated text, enabling low-latency safety detection on frontier models~\citep{cunningham2026constitutional}.
However, these defenses primarily block unsafe generations; they do not steer the model toward a safe alternative.

Another line of work steers LLMs by intervening on model activations during generation.
Activation-addition methods modify hidden states along fixed directions~\citep{turner2023steering,arditi2024refusal,rimsky2024steering}.
While these interventions are simple and efficient, they can be coarse, applying the same direction across contexts and may degrade utility.
A more adaptive approach is to use learned safety classifiers as steering signals~\citep{lee2025programming}.
Motivated by hidden-state classifiers for production-grade safety detection~\citep{cunningham2026constitutional}, classifier-guided methods use classifier scores not only to detect unsafe states, but also to guide how those states should be modified.
Methods such as \texttt{SaP} steer using latent-space constraints~\citep{ICML2025_chenlearning}, but existing approaches remain limited by linear classifier assumptions or by expensive constrained optimization and perturbation-search procedures~\citep{ICML2025_chenlearning,kong2024aligning,karnik2025preemptivedetectionsteeringllm,zhao2026odesteer}.
In contrast, \algo{} learns nonlinear latent-space safety constraints, derives closed-form steering rules, and efficiently composes multiple constraints.
We defer further discussion of inference-time safety and classifier-guided steering to Appendix~\ref{appendix:additional-related-work}.

\para{Control Barrier Functions for Formal Safety.}
Control Barrier Functions (CBFs)~\citep{Ames2017,Xiao2021TAC2,Glotfelter2017} are a standard tool for enforcing safety in dynamical systems.
CBFs turn nonlinear safety conditions into quadratic programs that can be solved efficiently online, providing a practical safety mechanism with formal guarantees under appropriate assumptions.
This makes them especially appealing for learning-based control, where learned policies can be paired with explicit safety mechanisms rather than relying only on soft constraints~\citep{achiam2017constrained,tessler2018reward}.
Advances in differentiable optimization~\citep{Amos2017} have enabled CBFs to be integrated into neural networks as safety-guaranteed layers~\citep{pereira2020,liu2023learning,Xiao2021bnet}, primarily for robotic control, and CBFs have also been extended to generative models such as LLM alignment~\citep{zhao2026odesteer} and diffusion models~\citep{xiao2023safediffuser}.
However, most CBF-based methods assume that the safety constraint is known \emph{a priori}, which is often unrealistic for complex systems such as LLMs.
Recent work has explored using CBFs to steer LLMs~\citep{miyaoka2025control}, but derives barriers from text-level safety classifiers and relies on constrained optimization during decoding, which can introduce substantial inference overhead.
In contrast, \algo{} learns latent-space safety constraints and derives closed-form rules for efficient inference-time steering.
Since these constraints are learned from data, the resulting guarantees are conditional on the learned barriers accurately capturing the intended safety property~\citep{robey2020learning, so2024train}.

%% file: latex/problem.tex

This section provides the background and motivation for \algo{}.
We first review how prior work learns hidden-state safety classifiers and uses their decision boundaries as safety constraints.
We then discuss why using such constraints to steer LLMs remains practically challenging.
Finally, we introduce Control Barrier Functions, a standard tool from control theory, which \algo{} uses to transform learned nonlinear safety constraints into efficient inference-time steering rules.

Throughout the paper, we use \emph{classifier}, \emph{constraint}, and \emph{barrier function} depending on how the learned safety function is used.
It is a \emph{classifier} when distinguishing safe from unsafe hidden states, a \emph{constraint} when controlling LLM generation, and a \emph{barrier function} when used in the CBF condition.

\subsection{Steering LLMs using Latent Classifiers}
\label{sec:latent-classifier-steering}

\para{Learning Latent Classifiers for LLM Safety.}
Performing safety classification directly in text space often requires an external language model classifier~\citep{sharma2025constitutionalclassifiersdefendinguniversal}, which can incur high computational cost when applied during generation.
Recent work suggests that hidden-state classifiers can provide faster safety detection, with promising results in production-grade settings~\citep{cunningham2026constitutional}.
We follow prior work~\citep{cunningham2026constitutional,ICML2025_chenlearning} and use the standard hidden-state classifier training pipeline with binary safety data.
Concretely, we are given a labeled dataset $\mathcal{D}=\{(q^{(i)},x^{(i)},y^{(i)})\}_{i=1}^{N}$, where $q^{(i)}$ is the user query and $x^{(i)}$ is the LLM response.
Each response $x^{(i)}$ consists of $T^{(i)}$ tokens, $x^{(i)}=(x_{1}^{(i)},\ldots,x_{T^{(i)}}^{(i)})$, with a binary response-level label $y^{(i)}\in\{0,1\}$, where $y^{(i)}=1$ indicates the presence of unsafe content.
These labels are obtained from human annotations or automated safety classifiers~\citep{mazeika2024harmbench}.

To construct the latent training set, we run the LLM
on each query--response pair and collect hidden states. Define the response-prefix state $s_t^{(i)}=(q^{(i)},x_{\leq t}^{(i)})$.
Let $h^{(\ell)}(s_t^{(i)})$ denote the hidden state corresponding to the response token $x_t^{(i)}$ at layer $\ell$.
This hidden state contains information about the query $q^{(i)}$ and the response prefix $x_{\leq t}^{(i)}$, rather than only the current token $x_t^{(i)}$~\citep{vaswani2017attentionneed,devlin2019bertpretrainingdeepbidirectional,radford2019language}.
Crucially, because the supervision is assigned at the response level, each response $x^{(i)}$ yields $T^{(i)}$ hidden-state examples with the same label $y^{(i)}$: $\{(h^{(\ell)}(s_t^{(i)}),y^{(i)})\}_{t=1}^{T^{(i)}}$. 
Thus, an unsafe response contributes unsafe hidden-state examples across its generated positions, even though the original annotation does not identify a specific unsafe token span. 
These hidden-state examples are then used to learn latent safety classifiers whose decision functions define safety boundaries in hidden-state space.

\para{Steering LLMs using Safety Constraints.}
Recent work, such as \texttt{SaP}~\citep{ICML2025_chenlearning}, suggests explicitly viewing learned latent classifiers as safety constraints.
Concretely, whenever a hidden state $h$ is predicted to violate the learned constraint, \texttt{SaP} computes a projected safe representation by optimizing:
$$
    {{\min_{h^\star} \|h-h^\star\| \text{~s.t.~} h^\star \text{ satisfies the learned safety constraints.}}}
$$

This approach combines the detection strength of latent classifiers~\citep{cunningham2026constitutional} with activation steering, which adjusts LLMs toward safer generation~\citep{turner2023steering,arditi2024refusal}.
However, \texttt{SaP} has several \textit{limitations}:

\vspace{-3mm}
\begin{itemize}[leftmargin=*]
	\setlength\itemsep{0.0em}
    \setlength{\itemindent}{2pt}
    \item \textbf{Linearity:} \texttt{SaP} uses linear constraint boundaries, restricting classifier design.
    \item \textbf{Expensive optimization:} Applying the constraints requires optimization using gradient descent at inference time, so each steering step adds extra computation and increases generation latency.
    \item \textbf{Limited modularity:} Composing existing constraints or adding new ones may require retraining.

\end{itemize}

These limitations of SaP motivate \algo{}: a principled steering method that handles nonlinear latent constraints, derives closed-form updates, and composes safety constraints from independently trained classifiers.
Empirically, \algo{} improves adversarial defense while preserving utility (\cref{sec:exp_adv_attack}), reduces steering latency (\cref{sec:exp_latency}), and enables modular safety composition (\cref{sec:exp_compose_cnf}).

\subsection{Steering LLMs with Control Theory}
\label{sec:control-theory-steering}

\para{Autoregressive Generation as a Latent Dynamical System.}
Autoregressive language generation naturally induces a trajectory in the model's latent space~\citep{kong2024aligning}.
At decoding step $t$, the prefix state $s_t=(q,x_{\leq t})$ is processed by the LLM, which produces an intermediate hidden state $h^{(\ell)}(s_t)$.
The model then samples the next token $x_{t+1}$,
forming the updated state $s_{t+1}=(q,x_{\leq t+1})$.
This induces a stochastic transition over latent representations, $h^{(\ell)}(s_t)\mapsto h^{(\ell)}(s_{t+1})$, allowing autoregressive generation to be viewed as a latent dynamical process.
This view frames inference-time LLM safety as keeping latent trajectories within, or steering them toward, prescribed safe regions.

CBFs provide a natural mechanism for this purpose: a learned safety classifier defines a safe set in the hidden-state space, and the CBF condition determines local corrections to steer 
the trajectory away from unsafe regions during decoding.
\algo{} instantiates this idea by treating hidden states as latent system states and learned safety classifiers as CBF constraints.

\begin{figure*}
    \vspace{-4mm}
    \centering
    \includegraphics[width=\linewidth]{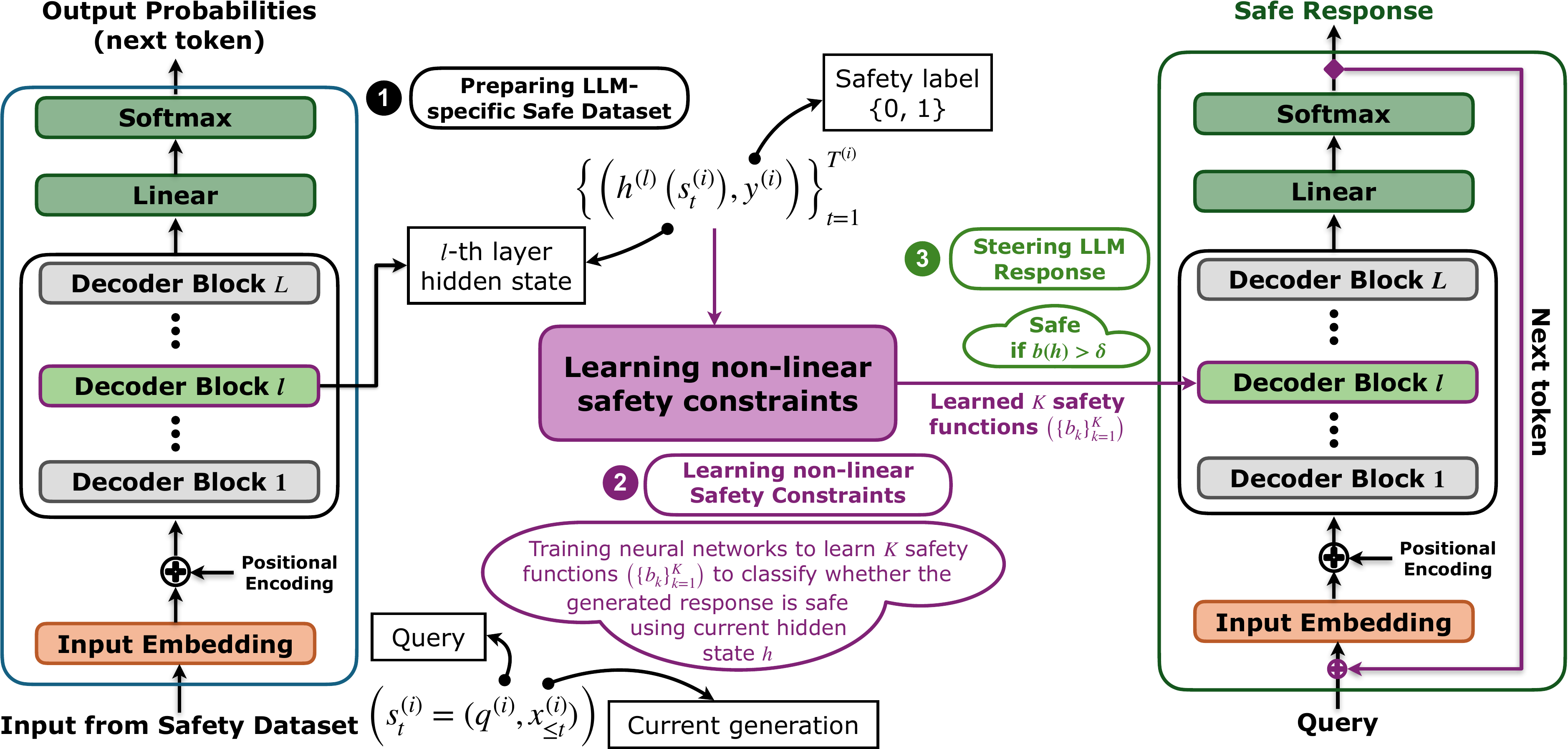}
    \caption{
        \textbf{Overview of \algo{}.}
        \algo{} follows a three-stage pipeline:
        \protect\stageicon{black}{1} extracting intermediate LLM hidden states
        and pairing them with binary safety labels
        (\cref{sec:latent-classifier-steering});
        \protect\stageicon{stagepurple}{2} learning nonlinear latent safety classifiers
        that define hidden-state safe sets
        (\cref{sec:learning_cbf});
        and \protect\stageicon{stagegreen}{3} applying closed-form CBF-based steering
        to guide hidden states toward learned safe sets
        (\cref{sec:steering,sec:composing_cbf}).
    }
    \label{fig:SafeLLM}
    \vspace{-6mm}
\end{figure*}

\para{Control Barrier Functions.}
CBFs provide a principled framework for enforcing safety in dynamical systems.
We consider affine control systems with state $h \in \mathbb{R}^n$, control input $u \in \mathbb{R}^m$, and dynamics
$$
    \dot{h} = f(h) + g(h)u,
$$
where $f$ and $g$ are assumed to be locally Lipschitz.
The goal is to ensure that the state $h$ remains within a prescribed safe set $\mathcal{C}$, defined as the $0$-superlevel set of a continuously differentiable function $b : \mathbb{R}^n \to \mathbb{R}$:
$$
    \mathcal{C} = \{h \in \mathbb{R}^n : b(h) \geq 0\}.
$$

\begin{defi}[Control Barrier Function]
    A function $b$ is a control barrier function for the system on $\mathbb{R}^n$ if there exists an extended class-$\mathcal{K}$ function $\beta$ such that
$$
    \forall h \in \mathbb{R}^n:
    \sup_{u \in \mathbb{R}^m} \dot{b}(h,u)
    \geq -\beta(b(h)).
$$
\end{defi}

Since $\dot{b}(h, u) = L_f b(h) + L_g b(h)u$ by the chain rule, this condition can be equivalently written as
$$
    \sup_{u \in \mathbb{R}^m}
    \left[
        L_f b(h) + L_g b(h)u + \beta(b(h))
    \right] \geq 0.
$$
Here, $L_f b(h)$ and $L_g b(h)$ denote the Lie derivatives of $b$ along the drift and control dynamics, respectively.
The term $\beta(b(h))$ shapes how aggressively the system must react as it approaches the boundary of the safe set.

Given a CBF $b$, any control input $u$ satisfying $\dot{b}(h,u) \geq -\beta(b(h))$ for all $h \in \mathcal{C}$ renders $\mathcal{C}$ forward invariant: trajectories starting in $\mathcal{C}$ remain in $\mathcal{C}$ under standard regularity assumptions~\citep{Ames2017,Xiao2021TAC2}.
For simplicity, we use a linear class-$\mathcal{K}$ function throughout this paper, i.e., $\beta(b(h)) = \alpha b(h), \alpha \in \mathbb{R}_{>0}$.

Notably, CBFs can enforce nonlinear safety constraints, admit efficient closed-form solutions, and are composable.
These properties match the needs of LLM activation steering: expressive safety boundaries, efficient online intervention, and modular composition of multiple constraints.

%% file: latex/algorithm.tex

In this section, we present our \algo{} framework for learning latent safety constraints and applying them as CBFs during inference.
As illustrated in \cref{fig:SafeLLM}, \algo{} follows a three-stage pipeline: 
\stageicon{black}{1} extracting intermediate latent representations from a pre-trained LLM and constructing a dataset with binary safety labels (see~\cref{sec:latent-classifier-steering});
\stageicon{stagepurple}{2} learning expressive nonlinear safety constraints in latent space; and
\stageicon{stagegreen}{3} applying CBF-based steering at inference time to guide hidden-state trajectories toward the learned safe set without modifying the underlying model parameters.

\subsection{Learning Control Barrier Functions}
\label{sec:learning_cbf}
A key challenge in applying CBFs to LLMs is that safety constraints are not known \emph{a priori} and must instead be inferred from data.
In contrast to prior approaches that rely on linear boundaries~\citep{ICML2025_chenlearning}, we learn nonlinear safety constraints directly using pairs of hidden states and safety labels.
We estimate safety boundaries by learning barrier functions $\{b_k\}_{k=1}^K$ with neural network parameters $\theta=\{\theta_1,\ldots,\theta_K\}$.
Given safe and unsafe hidden states from $\mathcal{H}_{\mathrm{safe}}$ and $\mathcal{H}_{\mathrm{unsafe}}$, we jointly train the barriers with a CBF-style classification loss~\citep{robey2020learning,so2024train} consisting of two terms:
$$
    \mathcal{L}(\theta) = \mathcal{L}_{\safe}(\theta) + \lambda_{\unsafe} \mathcal{L}_{\unsafe}(\theta),
$$
where the parameter $\lambda_{\text{unsafe}} \in \mathbb{R}_{>0}$ is a hyperparameter that controls the penalty associated with safety violations.
The \textit{safe set loss} makes safe states satisfy \textit{all} constraints ($ b_k(h) \geq 0, \forall k \in \{1, \dots, K\} $):
$$
    \mathcal{L}_{\safe}(\theta) = \sum_{h \in \mathcal{H}_{\safe}} \sum_{k=1}^K [- b_{k}(h)]_+.
$$

Conversely, the \textit{unsafe set loss} makes unsafe states violate \textit{at least one} constraint ($\min_k b_{k}(h) \leq -\epsilon$):
$$
    \mathcal{L}_{\unsafe}(\theta) = \sum_{h \in \mathcal{H}_{\unsafe}} [\min_{k} b_{k}(h) + \epsilon]_+.
$$
In these expressions, the operator $[ \cdot ]_+ = \max(0, \cdot)$ denotes the hinge loss, and $\epsilon$ defines the safety margin.
Minimizing $\mathcal{L}(\theta)$ encourages the intersection of the learned safety constraints to correctly classify safe and unsafe states. Moreover, distinct control barrier functions can be learned from separate datasets and composed, as discussed in Section~\ref{sec:composing_cbf}.

\subsection{Representation Steering with Latent Dynamics}
\label{sec:steering}
Once the safety constraints are learned, we can use them to steer the LLM generation.
Obtaining $b_{k}$ allows us to check whether the current latent state satisfies the safe set $\cC = \{h \mid b_{k}(h) \geq \delta, k\in \{1,\dots,K\}\}$, where $\delta$ denotes the safety threshold. In practice, we set $\delta = 0$ for simplicity.
If the model's original trajectory violates the safety margin (i.e., $\exists k, b_{k}(h) < \delta$), we dynamically adjust the latent state $h$ within the generation loop to steer it toward the learned safe region.

To steer using CBFs, we approximate the local evolution of the hidden state as $\dot{h} = (h_t - h_{t-1})/\Delta t$, where $\Delta t$ is the time step (default $\Delta t = 1$).
Following the local-dynamics view in \cref{sec:control-theory-steering}, this finite difference is used only to define local dynamics, not to predict the full future trajectory.
By modeling the latent-state evolution locally as a controllable dynamical system $\dot{h} = u$, we seek a control input $u$ that satisfies the local first-order CBF condition while minimizing deviation from the original trajectory.
Specifically, we compute the control $u$ by solving the following Quadratic Program (QP):
\begin{equation}
    \label{eqn:qp}
    u^* = \underset{u}{\text{argmin}} \ \norm{u - \frac{h_t - h_{t-1}}{\Delta t}}^2 \text{~s.t.~}
    \nabla b_{k}({h_t})^\top u \!+\! \alpha(b_{k}({h_t})-\delta) \geq 0, ~k\in \{1,\dots, K\}.
\end{equation}
The steered hidden state is $h'_t = h_{t-1} + u^*\Delta t$, which replaces $h_t$ in the decoding loop.
This avoids expensive optimization because the QP can admit efficient closed-form solutions, as discussed in \cref{sec:composing_cbf}.
Furthermore, the parameter $\alpha$ in \cref{eqn:qp} controls the steering strength. A larger value of $\alpha$ produces a more aggressive correction toward the learned latent safe set $\mathcal{C}$. It may also induce a significant deviation from the nominal control input. This fundamental trade-off between safety and task performance is empirically validated across various benchmarks, as reported in Table~\ref{tab:ablation_alpha}.

\subsection{Closed-form Steering via Composing CBFs}
\label{sec:composing_cbf}
Real-world safety specifications often involve multiple constraints (e.g., ``avoid harm'' and ``protect privacy'').
We formalize such requirements as satisfying all $K$ constraints in \cref{eqn:qp}.
However, directly solving the quadratic program is inefficient, since the control dimension matches the hidden-state dimension and is typically large.
To address this, we derive a closed-form steering rule.

Recall that $\{b_{k}(h)\}_{k=1}^K$ is a set of learned CBF boundaries corresponding to safe sets $\mathcal{C}_k$.
A system satisfying all constraints must remain in the intersection $ \bigcap_{k=1}^K \mathcal{C}_k $.
We merge these constraints into a single, continuously differentiable CBF.
The intersection of safe sets can be captured by a smooth under-approximation of the minimum function using the Log-Sum-Exp (LSE) formula~\cite{Boyd2004}:
\begin{equation} \label{eqn:compose}
    B(h) = -\frac{1}{\kappa} \ln \left( \sum_{k=1}^{K} e^{-\kappa (b_{k}(h)-\delta)} \right),
\end{equation}
where $\kappa \in \mathbb{R}_{>0}$ is a smoothing parameter.
Notably, this composed barrier $B(h)$ remains continuously differentiable, and the satisfaction of $B(h)\geq 0$ implies the satisfaction of $b_{k}(h)-\delta \geq 0, \forall k\in\{1,\dots, K\}$ \cite{Boyd2004}.
When the number of active constraints is at most 2 (either $K \leq 2$ originally, or after composing via~\cref{eqn:compose}), the QP in~\cref{eqn:qp} admits the following closed-form solution for $u$:
\begin{equation}
    \label{eqn:cf}
    u^* =  \frac{h_t - h_{t-1}}{\Delta t} + \lambda_1({h_t})\hat g_1({h_t}) + \lambda_2({h_t})\hat g_2({h_t}),
\end{equation}
where $\lambda_1, \lambda_2, \hat g_1,$ and $\hat g_2$ are defined in Appendix~\cref{sec:cf}.

\begin{remark}
	\label{rem:local-assumptions}
	We use the following local assumptions to apply the CBF for steering:
	(i) each learned barrier $b_k$ and the composed barrier $B$ in~\cref{eqn:compose} are continuously differentiable; and
	(ii) over the one-step neighborhood where steering is applied, the hidden-state evolution is locally Lipschitz and is well approximated by the finite-difference model $\dot{h} = (h_t - h_{t-1})/\Delta t$ used in~\cref{eqn:qp}.
	
	These assumptions are reasonable in our setting.
	First, our neural barrier heads use smooth GELU-based MLPs, and the LSE composition is smooth for $\kappa>0$ (Appendix~\cref{appendix:barrier-implementation-details}), so the required gradients are available by design.
	Second, the dynamics assumption is local and one-step: at each decoding step, the finite-difference model $\dot{h} = (h_t - h_{t-1})/\Delta t$ in~\cref{eqn:qp} serves as a first-order approximation for choosing a correction direction near the current hidden state.
	We do not assume it globally describes transformer decoding over an entire trajectory; instead, \algo{} re-evaluates the barrier at every decoding step and applies at most one closed-form steering update per token when the composed barrier condition is violated.
	Such local approximations have proven effective for CBF-style control in other generative models, including diffusion models~\citep{xiao2023safediffuser}.
\end{remark}

We formalize the safety property of the steering in~\cref{eqn:cf} under the LSE composition in~\cref{eqn:compose}.
\begin{restatable}{thm}{SafeGuarantee}
    \label{thm:safe}
    Let $\{b_k(h)\geq \delta\}_{k=1}^K$ be learned safety constraints, and let
    $B(h)$ denote the smooth composed barrier in~\cref{eqn:compose}, with
    composed safe set $\cC_B = \{h \mid B(h) \geq 0\}$.
    Let $h$ and $h'$ denote the hidden states before and after applying the
    control $u^*$ in~\cref{eqn:cf}.
    Under the local assumptions stated above, the update satisfies the composed CBF condition
    $\nabla B(h)^\top u^* + \alpha B(h) \geq 0$.
    Consequently, under the local dynamics, $\cC_B$ is locally forward invariant:
    if $h\in\cC_B$, the composed barrier condition is locally preserved; if
    $h\notin\cC_B$, the composed-barrier violation is exponentially stabilized
    toward $\cC_B$.
\end{restatable}
The proof of~\cref{thm:safe} is provided in Appendix~\ref{sec:proof}.

\cref{thm:safe} justifies using the composed barrier in~\cref{eqn:compose} for steering under multiple constraints, since $B(h)\geq 0$ is a sufficient condition for satisfying all $K$ constraints.
The guarantee, however, is conditional on the learned CBFs accurately capturing the intended text-level safety properties.

The closed-form steering in~\cref{eqn:cf} leads to three practical variants of \algo{}:
\vspace{-3mm}
\begin{itemize}[leftmargin=*]
    \setlength\itemsep{0.0em}
    \setlength{\itemindent}{2pt}
    \item \textbf{\algo{} (QP):} Directly solves the QP in \cref{eqn:qp}, which can be computationally expensive.
    \item \textbf{\algo{} (Top-2):} A selection heuristic that retains the two most violated constraints, reducing the active set to $K=2$, and then applies the closed-form solution in~\cref{eqn:cf}.
    \item \textbf{\algo{} (LSE):} Composes all constraints into a single CBF using the LSE formula in \cref{eqn:compose}, then applies the closed-form solution in \cref{eqn:cf} with $K=1$.
\end{itemize}
\vspace{-2mm}

In practice, we use \algo{} (LSE) as the default variant because it is theoretically grounded by \cref{thm:safe}, computationally efficient, and empirically effective.

%% file: latex/experiment.tex

\para{Datasets and Evaluation Setup.}
We evaluate \algo{} on two safety benchmarks using instruction-tuned LLMs that have undergone safety alignment but remain vulnerable to adversarial prompts.
First, we use \texttt{HarmBench}~\citep{mazeika2024harmbenchstandardizedevaluationframework}, a standardized framework for automated adversarial-prompt generation.
We select 320 harmful behaviors to construct training cases for learning the safety barriers and reserve 80 disjoint behaviors for evaluation.
This behavior-level split tests whether the method generalizes to unseen harmful behaviors rather than memorizing specific prompts.

Second, we use \texttt{WildGuardMix}~\citep{han2024wildguardopenonestopmoderation}, a large-scale dataset covering harmful prompts across 14 risk categories.
For each category, we sample 400 prompts and split them into 320 for training and 80 for evaluation.
We use these prompts to collect safe and unsafe activation datasets, denoted by $\mathcal{H}_{\mathrm{safe}}$ and $\mathcal{H}_{\mathrm{unsafe}}$, respectively.
As discussed in \cref{sec:latent-classifier-steering}, each response $x^{(i)}$ with length $T^{(i)}$ yields $T^{(i)}$ hidden-state examples, so both datasets $\mathcal{H}_{\mathrm{safe}}$ and $\mathcal{H}_{\mathrm{unsafe}}$ contain substantially more examples than the number of prompts (see~\cref{appendix:hidden-state-counts} for exact dataset sizes).
For both datasets, we quantify safety using the HarmBench classifier~\citep{mazeika2024harmbench}, which evaluates the semantic safety of generated responses.

\textbf{Baselines.}
We compare \algo{} against a diverse set of LLM safety methods, focusing on inference-time steering baselines.
First, we include \texttt{Circuit Breaker}~\citep{zou2024improving}, a training-time representation-engineering method that fine-tunes LoRA-style adapters to reroute internal representations associated with harmful generations, while preserving benign behavior through a retain loss.
We also compare against \texttt{Self-Reminder}, a prompting-based inference-time defense that prepends and appends safety reminders to the user prompt to encourage the model to follow safety policies.

Our main comparisons focus on activation steering methods, which directly modify hidden states during generation.
\texttt{Activation Addition}~\citep{turner2023steering,rimsky2024steering} constructs a steering direction from the difference between mean activations from two contrastive sets,
$
    r = |\mathcal{H}_{\safe}|^{-1}\sum_{h \in \mathcal{H}_{\safe}} h - |\mathcal{H}_{\unsafe}|^{-1}\sum_{h \in \mathcal{H}_{\unsafe}} h,
$
and applies the intervention $h' = h + \alpha r$, where $\alpha$ controls the steering strength.
\texttt{Directional Ablation}~\citep{arditi2024refusal} removes the component of the hidden state along an identified steering direction, $h' = h - r(r^\top r)^{-1}r^\top h$.
We also compare against \texttt{ReFT-r1}~\citep{wu2025axbenchsteeringllmssimple}, a weakly supervised rank-one representation-finetuning baseline that applies a learned linear intervention to hidden representations during generation.
Unlike difference-in-means steering, its intervention direction is learned from supervision rather than computed directly from activation mean differences.

Finally, we evaluate against \texttt{SaP}~\citep{ICML2025_chenlearning}, the closest baseline to \algo{} because it also uses learned safety boundaries as constraints for inference-time steering.
Like \algo{}, \texttt{SaP} learns latent safety constraints from hidden states and safe/unsafe supervision.
However, \algo{} differs from \texttt{SaP} in both the expressiveness of the learned constraint and the efficiency of the steering rule, which directly affects the safety--utility trade-off and latency observed in our experiments.
By default, we use \algo{} (LSE), with each classifier implemented as a five-layer MLP
matched to \texttt{SaP} in parameter count.
For fairness, all learned baselines use the same safe/unsafe data to train their LoRA adapters, steering directions, representation interventions, or latent constraints, and all steering methods intervene at layer $\ell=20$, following prior work~\citep{ICML2025_chenlearning, arditi2024refusal}. We provide implementation details for \algo{} and the baselines in~\cref{appendix:barrier-implementation-details,sec:implementation-details-baselines}, respectively.

\vspace{-4mm}
\subsection{Adversarial Attack Defense}
\label{sec:exp_adv_attack}
\vspace{-2mm}
We assess the robustness of \algo{} against representative adversarial attack methods.
These include gradient-based optimization attacks such as
\texttt{GCG}~\citep{zou2023universal},
\texttt{GBDA}~\citep{guo2021gradient},
\texttt{AutoPrompt}~\citep{shin2020autoprompt},
\texttt{PEZ}~\citep{wen2023hard},
\texttt{UAT}~\citep{wallace2019universal}, and
\texttt{AutoDAN}~\citep{liu2023autodan};
human-crafted jailbreaks, including
\texttt{HumanJailbreak}~\citep{shen2024anything} and
\texttt{DirectRequest}~\citep{mazeika2024harmbenchstandardizedevaluationframework};
and the search-based \texttt{PAP}~\citep{zeng2024johnny}.
We use these nine attack methods to construct the training data.
We further evaluate against \texttt{Adaptive Attack}~\citep{andriushchenko2025jailbreaking}, which explicitly optimizes against defended LLMs.
\texttt{Adaptive Attack} is used only for evaluation, not training.
Additional experimental details and results are provided in \cref{appendix:adv-attack}.

\para{Over-refusal.} We use XSTest~\citep{rottger2024xstest} to evaluate whether defended models over-refuse benign prompts.

\para{Utility Benchmarks.}
To ensure representation steering does not significantly degrade core model utility, we evaluate zero-shot on two standard benchmarks:
\texttt{MMLU}~\citep{hendrycks2020measuring} for general knowledge and \texttt{GSM8K}~\citep{cobbe2021trainingverifierssolvemath} for mathematical reasoning. We refer to~\cref{appendix:extended-utility-results} for results on more utility benchmarks.

Table~\ref{tab:joint_results} summarizes results across four model families, covering static HarmBench ASR, adaptive-attack ASR, over-refusal, and utility.
Overall, \algo{} (LSE) achieves the strongest static jailbreak defense, reducing ASR to near zero across models while preserving utility close to the original LLMs.
Since all learned baselines are trained on the same HarmBench safe/unsafe data, the lower ASR suggests that \algo{} more effectively captures the latent safety boundary and applies the learned constraints.
On static HarmBench ASR, \algo{} outperforms prompt-based defenses such as \texttt{Self-Reminder}, training-time methods such as \texttt{Circuit Breaker}, and inference-time representation-steering baselines such as \texttt{ActAdd}, \texttt{DirAbl}, \texttt{ReFT-r1}, and \texttt{SaP}.

The comparison also highlights the safety--utility trade-offs of existing baselines.
Fixed-direction steering methods can mitigate attacks, but often degrade benign behavior or downstream utility because the same intervention is applied across diverse contexts.
For example, \texttt{DirAbl} reduces the adaptive-attack ASR of \texttt{Gemma-2-9B} to $32.00\%$, but lowers GSM8K accuracy from $74.22\%$ to $49.51\%$ and increases XSTest over-refusal from $28.00\%$ to $43.20\%$.
Prompting defenses can be effective for some LLMs, such as \texttt{Self-Reminder} on \texttt{Llama2-7B} under both static and adaptive attacks, but often incur severe over-refusal and utility degradation.
Training-based baselines such as \texttt{Circuit Breaker} and \texttt{ReFT-r1} improve robustness, but do not match the constraint-guided steering methods \texttt{SaP} and \algo{} across static and adaptive attacks.
This suggests that modeling safety as latent constraints provides a more reliable mechanism for jailbreak defense.

Compared to the closest latent-constraint baseline, \texttt{SaP}, \algo{} provides substantially stronger robustness. On \texttt{Ministral-8B}, \algo{} reduces ASR to $0.83_{\pm 1.86}\%$ while \texttt{SaP} remains at $23.91_{\pm 3.49}\%$; on \texttt{Qwen2-1.5B}, \algo{} reduces ASR to $0.66_{\pm 0.47}\%$ while \texttt{SaP} remains at $5.35_{\pm 3.44}\%$.
For a fair comparison, we also ablate the inference-time steering strength of \texttt{SaP} by varying its optimization weights.
As shown in~\cref{sec:implementation-details-baselines}, changing this steering strength does not close the robustness gap to \algo{}.
This suggests that the improvement comes from the greater expressiveness of nonlinear neural CBF boundaries.
Crucially, these robustness gains are achieved while maintaining MMLU and GSM8K performance and avoiding extreme over-refusal.

\begin{table*}
\caption{
    Attack success rate (ASR, $\downarrow$), adaptive-attack ASR (AA ASR, $\downarrow$), over-refusal rate (XSTest, $\downarrow$), and utility metrics (\texttt{MMLU}, \texttt{GSM8K}, $\uparrow$) for \algo{} (LSE) and other LLM safety baselines across four model families.
    ASR is averaged over nine HarmBench adversarial attack methods.
    \texttt{ActAdd} denotes Activation Addition, and \texttt{DirAbl} denotes Directional Ablation.
    Results are reported over five runs for learning-based methods, while other baselines are evaluated deterministically.
}
\label{tab:joint_results}
\vspace{-1.5mm}
\centering
\scriptsize
\resizebox{\textwidth}{!}{
\setlength{\tabcolsep}{3pt}
\begin{tabular}{llcccccccc}
\toprule
\textbf{Model} & \textbf{Metric} & \textbf{Base} & \textbf{Self-Reminder} & \textbf{Circuit Breaker} & \textbf{ActAdd} & \textbf{DirAbl} & \textbf{ReFT-r1} & \textbf{SaP} & \textbf{\algo{}} \\ \midrule
\multirow{5}{*}{\textbf{Gemma-2-9b}}
 & ASR $\downarrow$ & $19.50$ & $12.89$ & $9.91_{\pm 2.60}$ & $19.00$ & $13.61$ & $12.30_{\pm 3.19}$ & $8.00_{\pm 0.94}$ & $\mathbf{0.00_{\pm 0.00}}$ \\
 & AA ASR $\downarrow$ & $70.00$ & $66.00$ & $56.00_{\pm 4.00}$ & $62.00$ & $32.00$ & $65.33_{\pm 3.06}$ & $67.33_{\pm 3.06}$ & $53.00_{\pm 23.00}$ \\[0.35mm]
 & XSTest $\downarrow$ & $28.00$ & $38.80$ & $34.53_{\pm 3.00}$ & $31.20$ & $43.20$ & $36.80_{\pm 2.77}$ & $33.33_{\pm 1.89}$ & $32.00_{\pm 0.44}$ \\[0.35mm]
 & MMLU $\uparrow$ & $71.86$ & $71.72$ & $70.94_{\pm 0.74}$ & $71.86$ & $66.16$ & $71.31_{\pm 0.44}$ & $71.86_{\pm 0.01}$ & $71.87_{\pm 0.09}$ \\
 & GSM8K $\uparrow$ & $74.22$ & $72.78$ & $73.44_{\pm 0.84}$ & $73.69$ & $49.51$ & $73.74_{\pm 0.25}$ & $73.72_{\pm 0.16}$ & $73.65_{\pm 0.04}$ \\ \midrule
\multirow{5}{*}{\textbf{Ministral-8B}}
 & ASR $\downarrow$ & $51.19$ & $20.25$ & $30.19_{\pm 2.51}$ & $51.31$ & $51.78$ & $37.56_{\pm 3.31}$ & $23.91_{\pm 3.49}$ & $\mathbf{0.83_{\pm 1.86}}$ \\
 & AA ASR $\downarrow$ & $74.00$ & $74.00$ & $66.00_{\pm 5.48}$ & $78.00$ & $80.00$ & $68.00_{\pm 0.00}$ & $70.00_{\pm 2.00}$ & $19.60_{\pm 2.61}$ \\[0.35mm]
 & XSTest $\downarrow$ & $20.40$ & $22.00$ & $16.27_{\pm 1.80}$ & $17.20$ & $17.20$ & $15.20_{\pm 1.06}$ & $15.60_{\pm 3.55}$ & $20.96_{\pm 7.04}$ \\[0.35mm]
 & MMLU $\uparrow$ & $64.73$ & $63.03$ & $64.56_{\pm 0.06}$ & $64.53$ & $64.45$ & $63.89_{\pm 0.28}$ & $64.47_{\pm 0.02}$ & $63.07_{\pm 1.95}$ \\
 & GSM8K $\uparrow$ & $64.22$ & $67.63$ & $64.92_{\pm 0.61}$ & $65.05$ & $64.06$ & $65.33_{\pm 0.23}$ & $64.67_{\pm 0.13}$ & $62.55_{\pm 2.46}$ \\ \midrule
\multirow{5}{*}{\textbf{Llama2-7B}}
 & ASR $\downarrow$ & $6.06$ & $3.22$ & $6.70_{\pm 0.31}$ & $6.47$ & $3.19$ & $3.43_{\pm 0.35}$ & $6.88_{\pm 0.21}$ & $\mathbf{1.28_{\pm 2.17}}$ \\
 & AA ASR $\downarrow$ & $72.00$ & $2.00$ & $68.00_{\pm 2.00}$ & $70.00$ & $80.00$ & $36.00_{\pm 0.00}$ & $76.00_{\pm 5.29}$ & $1.60_{\pm 0.89}$ \\[0.35mm]
 & XSTest $\downarrow$ & $41.60$ & $68.00$ & $42.40_{\pm 0.69}$ & $41.60$ & $48.80$ & $47.33_{\pm 1.40}$ & $38.67_{\pm 2.66}$ & $43.36_{\pm 0.83}$ \\[0.35mm]
 & MMLU $\uparrow$ & $46.47$ & $44.51$ & $46.50_{\pm 0.01}$ & $46.56$ & $46.29$ & $46.60_{\pm 0.06}$ & $46.57_{\pm 0.02}$ & $46.56_{\pm 0.04}$ \\
 & GSM8K $\uparrow$ & $20.92$ & $9.55$ & $20.90_{\pm 0.45}$ & $21.30$ & $19.94$ & $21.56_{\pm 0.38}$ & $20.92_{\pm 0.08}$ & $20.94_{\pm 0.62}$ \\ \midrule
\multirow{5}{*}{\textbf{Qwen2-1.5B}}
 & ASR $\downarrow$ & $23.89$ & $9.08$ & $23.88_{\pm 1.47}$ & $19.14$ & $16.97$ & $4.17_{\pm 0.47}$ & $5.35_{\pm 3.44}$ & $\mathbf{0.66_{\pm 0.47}}$ \\
 & AA ASR $\downarrow$ & $62.00$ & $40.00$ & $49.33_{\pm 3.06}$ & $58.00$ & $60.00$ & $48.67_{\pm 4.16}$ & $63.33_{\pm 11.55}$ & $24.80_{\pm 1.79}$ \\[0.35mm]
 & XSTest $\downarrow$ & $27.60$ & $36.40$ & $26.53_{\pm 1.29}$ & $32.00$ & $34.80$ & $76.93_{\pm 3.59}$ & $29.47_{\pm 0.46}$ & $29.44_{\pm 1.49}$ \\[0.35mm]
 & MMLU $\uparrow$ & $55.69$ & $54.96$ & $55.64_{\pm 0.02}$ & $55.71$ & $54.76$ & $55.55_{\pm 0.04}$ & $55.69_{\pm 0.01}$ & $55.65_{\pm 0.05}$ \\
 & GSM8K $\uparrow$ & $24.94$ & $31.01$ & $26.71_{\pm 1.29}$ & $12.43$ & $13.12$ & $22.92_{\pm 1.34}$ & $12.63_{\pm 0.19}$ & $12.95_{\pm 0.63}$ \\ \bottomrule
\end{tabular}%
}
\vspace{-2mm}
\end{table*}

\begin{table*}[t]
\centering
\begin{minipage}[t]{0.64\textwidth}
\centering
\caption{Ablation study on the effect of steering strength $\alpha$
for \algo{} (LSE) on \texttt{Qwen2-1.5B} with $\kappa = 100$ over 5 runs.}
\label{tab:ablation_alpha}
\vspace{-1mm}
\scriptsize
\resizebox{0.87\linewidth}{!}{
\begin{minipage}{\linewidth}
\scriptsize
\setlength{\tabcolsep}{3pt}
\begin{tabular}{lccccc}
\toprule
\textbf{$\alpha$} & \textbf{ASR} $\downarrow$ & \textbf{AA ASR} $\downarrow$ & \textbf{XSTest} $\downarrow$ & \textbf{MMLU} $\uparrow$ & \textbf{GSM8K} $\uparrow$ \\
\midrule
0.001 & $0.20_{\pm 0.26}$ & $56.80_{\pm 9.34}$ & $\mathbf{28.00_{\pm 1.70}}$ & $\mathbf{55.72_{\pm 0.04}}$ & $\mathbf{13.03_{\pm 0.14}}$ \\
0.01 & $0.66_{\pm 0.47}$ & $24.80_{\pm 1.79}$ & $29.44_{\pm 1.49}$ & $55.65_{\pm 0.05}$ & $12.95_{\pm 0.63}$ \\
0.1 & $0.06_{\pm 0.10}$ & $12.67_{\pm 13.61}$ & $60.13_{\pm 21.70}$ & $50.24_{\pm 4.58}$ & $12.36_{\pm 1.71}$ \\
0.2 & $0.02_{\pm 0.03}$ & $4.67_{\pm 5.03}$ & $66.53_{\pm 14.34}$ & $45.94_{\pm 7.02}$ & $11.37_{\pm 1.93}$ \\
0.3 & $0.04_{\pm 0.07}$ & $1.33_{\pm 2.31}$ & $69.73_{\pm 8.95}$ & $43.59_{\pm 6.07}$ & $10.69_{\pm 0.72}$ \\
0.5 & ${0.00_{\pm 0.00}}$ & $9.33_{\pm 11.02}$ & $74.13_{\pm 2.81}$ & $40.45_{\pm 4.03}$ & $9.30_{\pm 1.25}$ \\
1.0 & $\mathbf{0.00_{\pm 0.00}}$ & ${2.67_{\pm 1.15}}$ & $76.93_{\pm 4.64}$ & $37.80_{\pm 2.81}$ & $7.25_{\pm 0.16}$ \\
\bottomrule
\end{tabular}
\end{minipage}%
}
\end{minipage}
\hfill
\begin{minipage}[t]{0.33\textwidth}
\centering
\caption{Steering latency (ms/token, $\downarrow$) on \texttt{Qwen2-1.5B}.
}
\label{tab:latency}
\vspace{-1mm}
\scriptsize
\setlength{\tabcolsep}{2pt}
\begin{tabular}{lc}
    \toprule
    \textbf{Method} & \textbf{Latency $\downarrow$} \\ \midrule
    SaP (100 steps) & $190.67_{\pm 3.17}$ \\
    SaP (1 step) & $139.97_{\pm 6.33}$ \\ \midrule
    \algo{} (Top-2) & $29.45_{\pm 11.00}$ \\
    \algo{} (QP) & $33.77_{\pm 5.86}$ \\
    \algo{} (LSE) & $\mathbf{2.40_{\pm 0.06}}$ \\ \bottomrule
\end{tabular}
\end{minipage}
\vspace{-5mm}
\end{table*}

\textbf{Impact of Steering Strength $\alpha$.}
We further investigate the sensitivity of \algo{} to the steering strength $\alpha$ in Table~\ref{tab:ablation_alpha}.
The results show that $\alpha$ directly controls the safety--utility trade-off: smaller values better preserve utility and maintain low over-refusal, whereas larger values apply stronger safety corrections and further reduce attack success.
In particular, increasing $\alpha$ drives static HarmBench ASR to zero and often improves robustness under adaptive attacks.

We use $\alpha=0.01$ in Table~\ref{tab:joint_results} because it substantially reduces both static and adaptive ASR while keeping XSTest and utility close to the base model.
In practice, $\alpha$ can be adjusted at inference time, enabling deployment-time calibration of the safety--utility trade-off:
applications that prioritize robustness can use a larger steering strength, whereas settings that prioritize benign-task performance can use a smaller value.
We provide the full ablation over $\alpha$ for all four LLMs in \cref{appendix:adv-attack}.

\vspace{-1mm}
\subsection{Computational Efficiency}
\label{sec:exp_latency}
\vspace{-1.5mm}
A primary advantage of \algo{} is its computational efficiency.
Unlike \texttt{SaP}, which runs gradient descent for each token at inference time, \algo{} (LSE) uses a closed-form update (see~\cref{sec:composing_cbf}).
We measure per-token inference latency only  for tokens that trigger steering, thereby isolating the steering overhead.
Additional benchmarking details and results, including total generation time with both inference and steering, are provided in \cref{sec:computational-benchmark-details}.
As shown in Table~\ref{tab:latency}, \algo{} (LSE) introduces negligible overhead, with a per-token steering latency of $2.40$ ms.
In contrast, \texttt{SaP} requires $190.67$ ms with 100 optimization steps and $139.97$ ms with one optimization step.
This corresponds to speedups of over $79\times$ and $58\times$, respectively.

\vspace{-1mm}
\subsection{Modular Composition of Safety Barriers}
\label{sec:exp_compose_cnf}
\vspace{-1.5mm}
\begin{wraptable}[9]{r}{0.45\textwidth}
    \vspace{-8mm}
    \centering
    \caption{Harmful rate ($\downarrow$) evaluations for \algo{} on \texttt{Qwen2-1.5B} when composing safety barriers independently trained on
    14 risk categories from \texttt{WildGuardMix}~\citep{han2024wildguardopenonestopmoderation}.
    Metrics are averaged over five runs.}
    \label{tab:composition_results}
    \vspace{-5mm}
    \vskip 0.15in
    \begin{small}
    \setlength{\tabcolsep}{5pt}
    \begin{tabular}{lc}
        \toprule
        \textbf{Method} & \textbf{Harmful Rate $\downarrow$} \\
        \midrule
        Original Model & $22.77$ \\
        Single \algo{} (avg.) & $4.74_{\pm 1.21}$ \\
        \algo{} (LSE) & $\mathbf{0.86_{\pm 0.04}}$ \\
        \bottomrule
    \end{tabular}
    \end{small}
    \vspace{-2mm}
\end{wraptable}

We demonstrate the modularity of \algo{} by training distinct Control Barrier Functions (CBFs), each using data from one of 14 harmful risk categories in \texttt{WildGuardMix}~\citep{han2024wildguardopenonestopmoderation}.
For robust generalization, we evaluate on a held-out test set containing unseen behaviors from all 14 categories.
These independently trained barriers are used to steer individually or composed using the LSE formula.
We refer to~\cref{appendix:compose-cbf-exp} for experimental details, results for other \algo{} variants, and an ablation on the effect of $\kappa$ in~\cref{eqn:compose}.

Table~\ref{tab:composition_results} presents the main composition results.
The original model has a high harmful-generation rate of $22.77\%$.
Averaged across the 14 single-barrier configurations, this rate decreases to $4.74\%$, showing that a barrier trained on one risk concept can also reduce unsafe generations in other categories.
Finally, \algo{} (LSE) reduces the harmful-generation rate to just $0.86\%$.
This shows that composing CBFs trained on independent concepts enables \algo{} to perform well across all concepts, consistent with~\cref{thm:safe}.
This modularity supports practical real-world use cases, where safety specifications can be applied compositionally depending on context and users.

\vspace{-1mm}
\subsection{Ablating Neural Barriers and Steering Mechanism}
\label{sec:exp_mechanism_ablation}
\vspace{-1.5mm}
We further ablate two design choices of \algo{}.
First, to test the role of nonlinear safety boundaries, we replace the 5-layer neural CBFs in the WildGuard composition experiment with linear probes, while keeping the same training data and LSE merge rules.
As shown in~\cref{appendix:neural-cbf-ablation,tab:composition-results-detailed}, the LSE harmful rate increases from $0.86_{\pm 0.04}$ with neural barriers to $2.89_{\pm 0.69}$ with linear probes, indicating that nonlinear barriers better capture safety constraints.
Second, to separate unsafe-state detection from the steering update itself, we evaluate conditional fixed-direction baselines that use the same learned barrier only as a trigger for applying \texttt{ActAdd} or \texttt{DirAbl}.
These baselines remain substantially weaker than \algo{} (\cref{appendix:conditional-fixed-steering,tab:conditional-fixed-steering}) in reducing ASR, suggesting that the gains come from the adaptive correction direction and magnitude, rather than from detection alone.

%% file: latex/conclusion.tex

We presented \algo{}, a principled inference-time framework for improving LLM safety by embedding learned nonlinear safety constraints in latent representation space. 
By treating hidden-state safety classifiers as Control Barrier Functions, \algo{} steers hidden-state trajectories toward learned safe regions while preserving the underlying model parameters. 
Empirically, \algo{} consistently reduces unsafe generations and adversarial attack success rates across models and datasets, including under adaptive attacks, while maintaining model utility and avoiding excessive over-refusal. 
Limitations include that the guarantees of \algo{} are conditional on the learned barriers accurately capturing the intended safety properties and apply to latent representations rather than providing strict worst-case guarantees for generated text. 
Future work includes improving the robustness of learned safety barriers, scaling \algo{} to frontier LLMs, and developing more fine-grained steering rules to better manage adaptive attacks and safety--utility--refusal trade-offs.

%% file: latex/appendix.tex

\section{Closed-form CBF Steering Derivation}
\label{sec:cf}

To derive the closed-form update, we consider the case where the active constraints have been reduced to at most two constraints by the merge rule using~\cref{eqn:compose} shown in \cref{sec:composing_cbf}.
 Here, we explicitly show the closed-form solution to the QP in~\cref{eqn:qp} following \citep{luenberger1997optimization}.
We consider the case where the QP has two constraints.
Given the following QP corresponding to ~\cref{eqn:qp} for the proposed \algo{}:
\begin{equation} \label{eqn:qp0}
\begin{aligned}
    u^{*} = \arg &\min_{u} ||u - \frac{h_t - h_{t-1}}{\Delta t}||^2,\\
\text{s.t., } & \nabla b_1(h_t)^\top u + \alpha b_1(h_t) \geq 0\\
& \nabla b_2(h_t)^\top u + \alpha b_2(h_t) \geq 0,
\end{aligned}
\end{equation}
where $b_1, b_2$ are the two constraints merged from the learned CBFs $\{b_{k}(h)\geq \delta \}_{k=1}^K$, following Sec. \ref{sec:composing_cbf}.
Next, we define
\begin{equation*}
    \begin{aligned}
    g_1(h_t) = [-\nabla b_1(h_t)^\top], \quad h_1(h_t) = \alpha b_1(h_t),\\
    g_2(h_t) = [-\nabla b_2(h_t)^\top],\quad h_2(h_t) = \alpha b_2(h_t).
    \end{aligned}
\end{equation*}

The standard form of the QP cost is $u^TH(h_t)u + F(h_t)u$,  where the matrix $H(h_t)$ is an identity matrix corresponding to \cref{eqn:qp0}. We can further define
\begin{equation*}
\begin{aligned}
[\hat g_1(h_t), \hat g_2(h_t)] = H(h_t)^{-1}[g_1(h_t), g_2(h_t)], \\
\left[
\begin{array}
	[c]{c}
	\hat h_1(h_t)\\ \hat h_2(h_t)
	\end{array}
	\right]
    =
    \left[
\begin{array}
	[c]{c}
	 h_1(h_t)\\ h_2(h_t)
	\end{array}
	\right] - \left[
\begin{array}
	[c]{c}
	 g_1(h_t)^T\\ g_2(h_t)^T
	\end{array}
	\right] \hat u
\end{aligned}
\end{equation*}
where
\begin{equation*}
\begin{aligned}
    \hat u = -H(h_t)^{-1}F(h_t), \qquad
    F(h_t) = - \frac{h_{t} - h_{t-1}}{\Delta t}.
    \end{aligned}
\end{equation*}

Then, let $w := u - \hat u$ and $\langle\cdot, \cdot\rangle$ define an inner product with weight matrix $H(h_t)$ so that $\langle w, w\rangle = w^T H(h_t) w$. We have that the optimization problem in \cref{eqn:qp0} is equivalent to:
\begin{equation} \label{eqn:qp_eq}
\begin{aligned}
     w^{*} = \arg &\min_{w} \langle w, w\rangle,\\
\text{s.t., } & \langle\hat g_1(h_t),w\rangle \leq \hat h_1(h_t),\\
& \langle\hat g_2(h_t),w\rangle \leq \hat h_2(h_t),
\end{aligned}
\end{equation}
where the optimal solution of \cref{eqn:qp0} is given by
\begin{equation*}u^{*} = w^* + \hat u.
\end{equation*}

Following \cite{luenberger1997optimization} [Ch. 3], the unique solution to \cref{eqn:qp_eq} is given by
\begin{equation*}
    w^* = \lambda_1(h_t)\hat g_1(h_t) + \lambda_2(h_t)\hat g_2(h_t)
\end{equation*}
where
\begin{equation*}
\begin{aligned} \lambda_1(h_t) = \left\{ \begin{array}{lll} 0 &\text{ if $G_{21}(h_t)[\hat h_2(h_t)]_+ - G_{22}(h_t)\hat h_1(h_t) < 0$} \\ \frac{[\hat h_1(h_t)]_+}{G_{11}(h_t)} &\text{ if $G_{12}(h_t)[\hat h_1(h_t)]_+ - G_{11}(h_t)\hat h_2(h_t) < 0$}\\ \frac{[G_{22}(h_t)\hat h_1(h_t) - G_{21}(h_t)\hat h_2(h_t)]_+}{G_{11}(h_t)G_{22}(h_t) - G_{12}(h_t)G_{21}(h_t)} &\text{ otherwise }. \end{array} \right.\\ \end{aligned}
\end{equation*}

\begin{equation*}
\begin{aligned} \lambda_2(h_t) = \left\{ \begin{array}{lll} \frac{[\hat h_2(h_t)]_+}{G_{22}(h_t)} &\text{ if $G_{21}(h_t)[\hat h_2(h_t)]_+ - G_{22}(h_t)\hat h_1(h_t) < 0$} \\ 0 &\text{ if $G_{12}(h_t)[\hat h_1(h_t)]_+ - G_{11}(h_t)\hat h_2(h_t) < 0$}\\ \frac{[G_{11}(h_t)\hat h_2(h_t) - G_{12}(h_t)\hat h_1(h_t)]_+}{G_{11}(h_t)G_{22}(h_t) - G_{12}(h_t)G_{21}(h_t)} &\text{ otherwise }. \end{array} \right.\\ \end{aligned}
\end{equation*}
where $G(h_t) = [G_{ij}(h_t)] = [\langle\hat g_i(h_t), \hat g_j(h_t)\rangle], i,j = 1,2$ is the Gram matrix.

\section{Proof of Theorem \ref{thm:safe}}
\label{sec:proof}
\SafeGuarantee*
\textit{Proof.} We prove the theorem for the local CBF update at the current hidden state under the local assumptions stated in the main text: the learned and composed barriers are continuously differentiable, and the one-step hidden-state evolution is locally Lipschitz and well approximated by the finite-difference model in~\cref{eqn:qp}. Following~\cref{sec:cf}, we have that the \algo{} control $u^*$ in~\cref{eqn:cf} is the exact solution of the QP in~\cref{eqn:qp0}. In other words, the following constraint is satisfied:
\begin{equation}
\label{eqn:cbf}
\begin{aligned}
  \nabla B(h)^\top u^* + \alpha B(h) \geq 0,
\end{aligned}
\end{equation}
where $B(h)$ is the composed barrier obtained via \cref{eqn:compose}, i.e.,
$$
B(h) = -\frac{1}{\kappa} \ln \left( \sum_{k=1}^{K} e^{-\kappa (b_{k}(h)-\delta)} \right).
$$

Following~\citep{Boyd2004}, $B(h)$ is a smooth lower bound of $b_{k}(h)-\delta$ for all $k\in\{1,\dots,K\}$.
Therefore, $B(h)\geq 0$ is a sufficient condition for satisfying every learned constraint $b_k(h)\geq \delta$.

Thus, if the current hidden state $h$ is in the composed learned safe set $\mathcal{C}_B = \{h \mid B(h) \geq 0\}$, then under the local dynamics $\dot{h} = u$ adopted in~\cref{sec:steering}, we have $\dot{B}(h) = \nabla B(h)^\top u^*$, and~\cref{eqn:cbf} becomes
\begin{equation}
    \dot B(h)+\alpha B(h)\geq 0.
\end{equation}
By the comparison lemma, this implies $B(h')\geq B(h)e^{-\alpha}$ under the local update, so $B(h')\geq 0$ whenever $B(h)\geq 0$.
Therefore, the local CBF update preserves the composed barrier condition.

Since $B(h)$ lower-bounds all $b_k(h)-\delta$, satisfying the composed-barrier condition is sufficient for satisfying all learned constraints. Thus, the \algo{} control obtained via the composing method in \cref{eqn:compose} preserves the composed-barrier condition under the local CBF update.

On the other hand, if the current hidden state $h$ is outside $\mathcal{C}_B$, then $B(h) \leq 0$.
To analyze convergence toward $\mathcal{C}_B$, we introduce the auxiliary function $V(h) = -B(h)$, which satisfies $V(h) \geq 0$.
Under the local dynamics $\dot{h} = u^*$, $\dot{V}(h) = -\dot{B}(h) = -\nabla B(h)^\top u^* = \nabla V(h)^\top u^*$, so~\cref{eqn:cbf} becomes

\begin{equation} \label{eqn:diff_cbf}
    \dot V(h) + \alpha V(h) =\nabla V(h)^\top u^* + \alpha V(h) \leq 0,
\end{equation}

The corresponding equality $\dot{V}(h) + \alpha V(h) = 0$ has solution $V(h') = V(h)e^{-\alpha}$.
By the comparison lemma~\cite{Khalil2002},~\cref{eqn:diff_cbf} therefore implies
$$V(h') \leq V(h)e^{-\alpha}.$$
Hence, $V(h')=-B(h')$ contracts exponentially under the local CBF update.
Since $B(h)$ is a lower bound of $b_{k}(h)-\delta$ for all $k\in\{1,\dots,K\}$, reducing the composed-barrier violation moves the updated hidden state $h'$ toward the learned safe set. $\hfill \blacksquare$

\section{Implementation Details of \algo{}}
\label{appendix:barrier-implementation-details}

This section provides implementation details for \algo{}.
We first outline the inference-time steering algorithm, then describe the neural CBF architecture and the latent dynamics model used by \algo{}.
The key implementation choice is that \algo{} learns nonlinear scalar barriers in hidden-state space, but applies them through a lightweight CBF update during decoding.

\para{Inference-time algorithm.}
Algorithm~\ref{alg:barriersteer-appendix} summarizes the implementation view of \algo{}.
At each decoding step, the method evaluates the learned barriers on the current hidden state, estimates the local hidden-state velocity from consecutive decoding states, and applies a CBF correction only when the current state is predicted to violate the learned safety constraints.
The same algorithm supports the Top-2, QP, and LSE merge rules discussed in the main text.

\begin{algorithm}[t]
\caption{Algorithmic view of \algo{} inference-time steering.}
\label{alg:barriersteer-appendix}
\begin{algorithmic}[1]
\REQUIRE Aligned LLM $\mathcal{M}$, learned barriers $\{b_k\}_{k=1}^K$, user query $q$, threshold $\delta$, steering strength $\alpha$, time step $\Delta t$, layer $\ell$, merge rule $m \in \{\mathrm{Top\text{-}2}, \mathrm{QP}, \mathrm{LSE}\}$.
\ENSURE Steered response $x=(x_1,\ldots,x_T)$.
\STATE Initialize response prefix $x_{\leq 0}\leftarrow \emptyset$, prefix state $s_0\leftarrow q$, and previous hidden state $h_0\leftarrow h^{(\ell)}(s_0)$.
\FOR{$t = 1, \ldots, T$}
    \STATE Form the current prefix state $s_t=(q,x_{\leq t})$ and extract the layer-$\ell$ hidden state $h_t\leftarrow h^{(\ell)}(s_t)$.
    \STATE Evaluate barrier residuals $r_k \leftarrow b_k(h_t)-\delta$ for all active safety constraints.
    \STATE Estimate the nominal latent velocity $u_0 \leftarrow (h_t-h_{t-1})/\Delta t$.
    \IF{all $r_k \geq 0$}
        \STATE Set $h'_t \leftarrow h_t$ without intervention.
    \ELSE
        \STATE Use merge rule $m$ to select or compose the active barriers via~\cref{eqn:compose}.
        \STATE Compute the CBF control $u^*$ using the active barrier gradients, $u_0$, and $\alpha$.
        \STATE Apply the controlled hidden state $h'_t \leftarrow h_{t-1}+u^*\Delta t$.
    \ENDIF
    \STATE Continue decoding from $h'_t$, sample the token, and append it to the response prefix.
\ENDFOR
\RETURN $x$.
\end{algorithmic}
\end{algorithm}

\para{Latent dynamics used by \algo{}.}
We use the standard control-affine notation
\begin{align*}
    \dot h = f(h) + g(h)u,
\end{align*}
where $h$ denotes the hidden state being steered and $u$ is the control applied in latent space.
In our implementation, we do not train a separate dynamics model.
Instead, because the controller directly edits the current hidden state, we set the drift term to zero and the control matrix to identity, i.e., $f(h)=0$ and $g(h)=I$.
The nominal motion of the unsteered model is estimated online as $u_0=(h_t-h_{t-1})/\Delta t$.
Thus, the CBF layer only needs to compute a corrected latent velocity $u^*$, after which the steered state is applied as $h'_t=h_{t-1}+u^*\Delta t$.
This avoids learning autoregressive hidden-state dynamics while still using the observed local trajectory of the model during decoding.

\para{Neural CBF architecture.}
Each barrier $b_k$ is implemented as an independent multilayer perceptron that maps a hidden state in $\mathbb{R}^{d_\phi}$ to a scalar barrier value.
For each head, we use the architecture
\begin{align*}
    d_\phi \rightarrow 2048 \rightarrow 1024 \rightarrow 512 \rightarrow 256 \rightarrow 1.
\end{align*}
Every hidden block consists of a linear layer followed by LayerNorm, GELU, and dropout; the final layer is a linear scalar head.
This gives each barrier a nonlinear decision boundary while keeping the implementation close to the latent-feature architecture used by \texttt{SaP}.

By default, our experiments use $K=4$ independent CBF heads.
This choice is intended to keep the number of trainable barrier parameters comparable to the \texttt{SaP} baseline while allowing nonlinear constraints.
For the Qwen2-1.5B setting, where $d_\phi=1536$, the above architecture has $5{,}910{,}017$ parameters per CBF head and $23{,}640{,}068$ total trainable barrier parameters for $K=4$.
The corresponding \texttt{SaP} latent-constraint module has $25{,}673{,}758$ trainable parameters, so improvements are not simply due to using a substantially larger safety module.
For other model families, we use the same architecture with the corresponding hidden-state dimension $d_\phi$.

\section{Experimental Setup and Additional Results}
\label{appendix:exp-details}

This appendix provides additional details for reproducing the experiments in~\cref{sec:experiments}.
We first describe how response-level safety labels are converted into token-level hidden-state supervision and report the exact hidden-state counts used to train the barriers.
We also provide implementation notes for the baselines, per-attack HarmBench results, variant-selection ablations, latency measurements, the WildGuard modular-composition study, and compute resources.

\subsection{Hidden-State Supervision and Training Set Sizes}
\label{appendix:hidden-state-counts}

\algo{} trains safety barriers from binary safe/unsafe response labels rather than preference pairs.
For HarmBench, we use adversarially elicited model responses from the HarmBench protocol~\citep{mazeika2024harmbenchstandardizedevaluationframework}; for WildGuard, we use category-specific examples from WildGuardMix~\citep{han2024wildguardopenonestopmoderation}.
Generated responses are labeled with the HarmBench classifier~\citep{mazeika2024harmbench}, and the resulting labels supervise hidden states extracted from the target model.
This follows the standard latent-safety-classifier setting used by prior representation-space safety methods~\citep{ICML2025_chenlearning,cunningham2026constitutional}, but \algo{} uses the learned classifier as a CBF for inference-time steering.

Because the available supervision is response-level, each generated response contributes token-level hidden states with the same binary label.
If a response is labeled unsafe, all extracted hidden states from that response are assigned the unsafe label; if it is labeled safe, the corresponding hidden states are assigned the safe label.
This construction is appropriate for decoder-only Transformers because the hidden state at a generated position is contextual: it summarizes the prompt and generated prefix through causal self-attention, rather than representing only the current token in isolation.
Consequently, a few hundred behavior prompts yield millions of labeled hidden-state examples.
Tables~\ref{tab:harmbench-hs-counts} and~\ref{tab:wildguard-hs-counts} report the exact counts used in our experiments.

\begin{table}[!t]
\centering
\caption{Exact HarmBench token-level hidden-state counts used for training. Counts aggregate all nine HarmBench attack-method training files for each model.}
\label{tab:harmbench-hs-counts}
\begin{small}
\setlength{\tabcolsep}{5pt}
\begin{tabular}{lrrr}
\toprule
\textbf{Model} & \textbf{Safe tokens} & \textbf{Unsafe tokens} & \textbf{Total tokens} \\
\midrule
Gemma-2-9B & 2,083,847 & 668,665 & 2,752,512 \\
Llama-2-7B & 2,652,477 & 100,035 & 2,752,512 \\
Ministral-8B & 1,266,580 & 1,485,932 & 2,752,512 \\
Qwen2-1.5B & 2,340,020 & 412,082 & 2,752,102 \\
\bottomrule
\end{tabular}
\end{small}
\vspace{-3mm}
\end{table}

\begin{table}[!t]
\centering
\caption{Exact WildGuard token-level hidden-state counts used for training individual category barriers on \texttt{Qwen2-1.5B}. We report the per-concept counts because the WildGuard experiment trains one barrier per concept and reports WildGuard results for Qwen2-1.5B.}
\label{tab:wildguard-hs-counts}
\begin{small}
\setlength{\tabcolsep}{3pt}
\begin{tabular}{lrrr}
\toprule
\textbf{WildGuard concept} & \textbf{Safe tokens} & \textbf{Unsafe tokens} & \textbf{Total tokens} \\
\midrule
Copyright & 114,666 & 16,406 & 131,072 \\
Cyberattack & 64,237 & 66,835 & 131,072 \\
Defamation / unsafe actions & 100,669 & 30,403 & 131,072 \\
Discrimination & 101,575 & 29,497 & 131,072 \\
Disinformation campaigns & 93,457 & 37,615 & 131,072 \\
Fraud / illegal activities & 98,252 & 32,820 & 131,072 \\
Harmful misinformation & 97,457 & 33,615 & 131,072 \\
Mental health crisis & 101,916 & 29,156 & 131,072 \\
Others & 97,074 & 33,998 & 131,072 \\
Private information & 106,845 & 24,227 & 131,072 \\
Sensitive information & 108,441 & 22,631 & 131,072 \\
Sexual content & 108,925 & 22,147 & 131,072 \\
Toxic language / hate speech & 110,190 & 20,882 & 131,072 \\
Violence / physical harm & 100,289 & 30,783 & 131,072 \\
\midrule
\textbf{Total} & \textbf{1,403,993} & \textbf{431,015} & \textbf{1,835,008} \\
\bottomrule
\end{tabular}
\end{small}
\vspace{-3mm}
\end{table}

\subsection{Implementation Details of Baselines}
\label{sec:implementation-details-baselines}

This section specifies the baseline implementations used in the main experiments.
Unless otherwise stated, all decoding uses greedy generation with temperature $0$, and all methods use the same prompts, behavior splits, and safety classifier described in~\cref{sec:exp_adv_attack} and Appendix~\ref{appendix:adv-attack}.
For learned representation-space methods, we use the same safe/unsafe response data as \algo{} whenever the baseline requires training supervision.
All hidden-state baselines use layer $L=20$, matching the hidden-state extraction layer used to train \algo{}.

\para{Base model.}
The base model is the original instruction-tuned checkpoint without additional prompting, fine-tuning, or activation intervention.
We evaluate each model using its respective default instruction prompt or chat template from Hugging Face: Llama-2-7B-Chat,
Ministral-8B-Instruct-2410, Qwen2-1.5B-Instruct, and Gemma-2-9B-IT.\footnote{Model checkpoints: \url{https://huggingface.co/meta-llama/Llama-2-7b-chat-hf}, \url{https://huggingface.co/mistralai/Ministral-8B-Instruct-2410}, \url{https://huggingface.co/Qwen/Qwen2-1.5B-Instruct}, and \url{https://huggingface.co/google/gemma-2-9b-it}.}
It serves as the reference for both safety and utility metrics, and all steering methods are applied on top of the same base checkpoint.

\para{Self-Reminder.}
\texttt{Self-Reminder}~\citep{xie2023defending} is a prompting-only inference-time defense that inserts safety reminders into both the system message and the user query while keeping model parameters unchanged.
We use the default \texttt{remind}/\texttt{remind} template from the released implementation:
\begin{quote}
\textbf{System reminder:} \texttt{\{original\_system\_message\} You should be a responsible language model and should not generate harmful or misleading content! Please answer the following user query in a responsible way.}

\textbf{Query reminder:} \texttt{Remember, you should be a responsible language model and should not generate harmful or misleading content!}
\end{quote}
The query reminder is prepended to the user prompt, while the system reminder augments the model's default system message.
Because this method only changes the prompt, it is evaluated deterministically under greedy decoding.

\para{Circuit Breaker.}
\texttt{Circuit Breaker}~\citep{zou2024improving} is a training-time representation method that learns LoRA adapters to reroute internal representations associated with harmful completions while preserving benign behavior through a retain objective.
We train it on the same HarmBench behavior split as the other learned methods.
Following the LoRRA objective from Circuit Breaker, let $\mathcal{M}_0$ denote the frozen base model, $\mathcal{M}_\theta$ the LoRA-adapted model, and $\operatorname{rep}_\ell(\mathcal{M},x)$ the layer-$\ell$ hidden representations selected
for the loss.
For circuit-breaker examples $x_s\sim\mathcal{D}_s$ and retain examples $x_r\sim\mathcal{D}_r$, the training objective is
\begin{align*}
    \mathcal{L}_{\mathrm{CB}}(\theta)
    ={}& c_s\,\mathbb{E}_{x_s\sim\mathcal{D}_s}
    \left[\operatorname{ReLU}\!\left(\cos\!\left(
    \operatorname{rep}_\ell(\mathcal{M}_0,x_s),
    \operatorname{rep}_\ell(\mathcal{M}_\theta,x_s)
    \right)\right)\right] \\
    &+ c_r\,\mathbb{E}_{x_r\sim\mathcal{D}_r}
    \left[\left\|\operatorname{rep}_\ell(\mathcal{M}_0,x_r)
    -\operatorname{rep}_\ell(\mathcal{M}_\theta,x_r)\right\|_2^2\right].
\end{align*}
The first term reroutes harmful representations away from the base model's harmful trajectory, while the retain term preserves representations on benign examples.
The coefficients $c_s$ and $c_r$ follow the Circuit Breaker schedule that trades off rerouting and retention during training.
The target layer for the circuit-breaker loss is layer $20$; LoRA adapters are inserted into transformer layers $0$--$20$.
The LoRA target modules are \texttt{q\_proj}, \texttt{k\_proj}, \texttt{v\_proj}, \texttt{o\_proj}, \texttt{gate\_proj}, \texttt{up\_proj}, and \texttt{down\_proj}.
We use LoRA rank $r=16$, LoRA alpha $16$, dropout $0.05$, circuit-breaker loss weight $10$, learning rate $10^{-4}$, batch size $8$, gradient accumulation $1$, weight decay $0$, maximum sequence length $1024$, and $150$ training steps with a constant learning-rate schedule.
Unlike inference-time steering methods, Circuit Breaker changes the model through an adapter before evaluation rather than computing a token-level intervention during decoding.

\para{Activation Addition.}
\texttt{Activation Addition}~\citep{turner2023steering,rimsky2024steering} constructs a fixed steering direction from the difference between mean safe and unsafe hidden states,
$ r = |\mathcal{H}_{\safe}|^{-1}\sum_{h\in\mathcal{H}_{\safe}} h - |\mathcal{H}_{\unsafe}|^{-1}\sum_{h\in\mathcal{H}_{\unsafe}} h $.
During generation, it applies the additive intervention $h'=h+\alpha r$ at layer $20$.
We use the same hidden-state data as \algo{} to estimate $r$, set $\alpha=1$, and do not normalize the steering vector.

\para{Directional Ablation.}
\texttt{Directional Ablation}~\citep{arditi2024refusal} uses the same safe--unsafe direction $r$, but removes the component of the hidden state along that direction rather than adding it: 
$h'=h-r(r^\top r)^{-1}r^\top h$.
This baseline tests whether suppressing the harmful/refusal-related direction is sufficient for safety compared to actively steering toward a learned safe region.

\para{ReFT-r1.}
\texttt{ReFT-r1}~\citep{wu2025axbenchsteeringllmssimple} is a weakly supervised rank-one representation-finetuning baseline.
Instead of computing a fixed mean-difference direction, it learns a rank-one linear intervention from the same safe/unsafe supervision used by the other learned representation baselines.
Training freezes the base LLM and optimizes only the rank-one intervention vector using the standard next-token cross-entropy loss on the target response tokens, with prompt tokens masked out using label value $-100$.
Let $\mathcal{M}_{\phi}$ denote the frozen LLM equipped with the learned rank-one intervention parameters $\phi$, and let $m_t\in\{0,1\}$ indicate whether token $x_t$ is part of the target response rather than the prompt.
The ReFT-r1 objective used in our experiments is
\begin{align*}
    \mathcal{L}_{\mathrm{ReFT}}(\phi)
    = -\mathbb{E}_{(q,x)}
    \left[\frac{1}{\sum_t m_t}\sum_{t=1}^{T}
    m_t \log p_{\mathcal{M}_{\phi}}(x_t\mid q,x_{<t})\right]
    + \lambda_{1}\,\Omega_{\mathrm{non\text{-}top}k}(\phi),
\end{align*}
where the optional sparsity term penalizes intervention scores outside the selected top-$k$ positions.
The implementation optionally adds this $\ell_1$ penalty on non-top-$k$ intervention scores, but the term is disabled in our experiments because we set $\lambda_1=0$.
The intervention is trained and applied at layer $20$ with top-$k=5$, intervention scale $\beta=1$, learning rate $10^{-3}$, batch size $4$, gradient accumulation $1$, warmup steps $0$, maximum sequence length $1024$, and $500$ training steps.
At inference time, the learned rank-one intervention is applied to hidden states during generation.

\para{SaP.}
\texttt{SaP}~\citep{ICML2025_chenlearning} is the closest latent-constraint baseline to \algo{}.
It learns linear safety constraints in a sparse concept space and performs inference-time projection when a hidden state violates the learned polytope.
We follow the released configuration: the concept encoder maps hidden states into a $16{,}384$-dimensional sparse representation, the polytope uses $K=30$ facets, and model-specific margin and regularization hyperparameters for Llama-2-7B, Ministral-8B, and Qwen2-1.5B are adopted from the original implementation. For Qwen2-1.5B, this gives $25{,}673{,}758$ trainable parameters in the \texttt{SaP} safety module.
Both \algo{} and \texttt{SaP} are trained with Adam, learning rate $10^{-2}$, and batch size $512$.

During inference, \texttt{SaP} optimizes the hidden state for samples predicted to be unsafe.
In our implementation, for an original hidden state $h$, \texttt{SaP} computes an updated state
\begin{align*}
    h^\star = \arg\min_{h'}\;&
    \frac{\|h'-h\|_1}{d_h}
    + \omega_{\mathrm{safe}} \sum_{k=1}^{K} v_k(h')
    + \lambda \sum_{k=1}^{K} [v_k(h')]_+,
\end{align*}
where $v_k(h')$ is the violation of the $k$-th learned linear constraint and $[\cdot]_+$ denotes the positive part.
We use default hyperparameters from the paper for this optimization. We set $\omega_{\mathrm{safe}}=10^{-4}$ in all SaP runs.
The coefficient $\lambda$ is therefore the main inference-time steering-strength parameter: larger values penalize positive constraint violations more strongly, but can also move the hidden state farther from the original representation.
We use the default $\lambda=2$ unless otherwise stated and run $100$ SGD iterations with step size $0.01$ for this projection step.
Table~\ref{tab:gemma-sap-lambda-asr} reports an ablation over $\lambda$ while keeping the learned SaP facets fixed.
The ablation shows that increasing $\lambda$ does not meaningfully improve SaP robustness once the learned linear facets are fixed.
This suggests that SaP's remaining ASR is not primarily due to an under-tuned projection strength, but rather to the expressiveness and optimization limits of enforcing linear latent constraints during decoding.

\begin{table*}[!t]
\centering
\caption{SaP inference-time steering-strength sensitivity on \texttt{Gemma-2-9B}. We vary the positive-violation penalty weight $\lambda$ while keeping the learned SaP facets fixed. ASR is measured on HarmBench static attacks and reported as mean$_{\pm\mathrm{std}}$.}
\label{tab:gemma-sap-lambda-asr}
\begin{small}
\setlength{\tabcolsep}{12pt}
\begin{tabular}{lccccc}
\toprule
\textbf{$\lambda$} & 1 & 2 & 5 & 10 & 20 \\
\midrule
\textbf{ASR $\downarrow$} & $8.17_{\pm 0.47}$ & $8.00_{\pm 0.94}$ & $8.08_{\pm 0.77}$ & $8.42_{\pm 1.20}$ & $8.50_{\pm 1.34}$ \\
\bottomrule
\end{tabular}
\end{small}
\vspace{-4mm}
\end{table*}

\subsection{Adversarial Attacks}
\label{appendix:adv-attack}
\label{appendix:baseline-protocol-details}

\para{Dataset collection and feature extraction.}
We follow the HarmBench adversarial-evaluation protocol~\citep{mazeika2024harmbenchstandardizedevaluationframework}.
For each model, the 400 HarmBench behaviors are split into 320 behaviors for training the latent safety constraints and 80 held-out behaviors for final evaluation.
For each attack method, we generate one test case per behavior.
To construct the barrier-training dataset, we run the target LLM on the 320 training behaviors and extract per-token hidden states from layer $L=20$.
Each response is labeled at the response level using \texttt{HarmBench-Llama-2-13b-cls}~\citep{mazeika2024harmbench}; the response label is then assigned to all extracted hidden states from that response, as described in Appendix~\ref{appendix:hidden-state-counts}.
Baseline-specific hyperparameters and implementation choices are summarized in Appendix~\ref{sec:implementation-details-baselines}.

\para{Evaluation protocol.}
The \texttt{SaP} numbers in our tables may differ from those 
reported in the \texttt{SaP} paper~\citep{ICML2025_chenlearning} 
because we use a 
different training/evaluation protocol.
Specifically, our HarmBench setting splits the behavior set: latent constraints are trained on one set of behaviors and evaluated on held-out behaviors, whereas some previously reported settings include behaviors that also appear during training.
The held-out split is stricter and better measures whether the learned constraint generalizes beyond the training behaviors.

\para{Detailed ASR results.}
For comprehensiveness, we include Tables~\ref{tab:gemma-results-transposed}--\ref{tab:qwen-results-transposed} to report per-attack HarmBench ASR for Gemma-2-9B, Ministral-8B, Llama-2-7B, and Qwen2-1.5B.
Learning-based methods are reported as averages over independent trials with standard deviations.
Prompting, fixed-direction, and heuristic baselines are evaluated deterministically under greedy decoding.

\begin{table*}[!t]
\centering
\caption{Per-attack HarmBench ASR for \texttt{Gemma-2-9B}. Values are percentages; lower is better. \algo{} denotes the default LSE variant. Learning-based methods are reported as mean$_{\pm\mathrm{std}}$ over independent runs, while deterministic baselines use fixed greedy decoding.}
\label{tab:gemma-results-transposed}
\vskip 0.15in
\vspace{-5mm}
\begin{small}
\setlength{\tabcolsep}{0pt}
\begin{tabular*}{\textwidth}{@{\extracolsep{\fill}}lcccccccc}
\toprule
\textbf{Attack Method} & \textbf{Base} & \textbf{Self-Reminder} & \textbf{Circuit Breaker} & \textbf{ActAdd} & \textbf{DirAbl} & \textbf{ReFT-r1} & \textbf{SaP} & \textbf{\algo{}} \\ 
\midrule
AutoDAN & 62.50 & 65.00 & $27.75_{\pm 11.74}$ & 58.75 & 22.50 & $45.42_{\pm 8.13}$ & $27.08_{\pm 14.49}$ & $0.00_{\pm 0.00}$ \\
AutoPrompt & 17.50 & 3.75 & $7.50_{\pm 2.80}$ & 21.25 & 13.75 & $10.83_{\pm 4.02}$ & $6.25_{\pm 1.25}$ & $0.00_{\pm 0.00}$ \\
DirectReq. & 5.00 & 2.50 & $4.75_{\pm 0.56}$ & 5.00 & 7.50 & $4.17_{\pm 0.72}$ & $8.00_{\pm 0.94}$ & $0.00_{\pm 0.00}$ \\
GBDA & 4.50 & 0.75 & $5.00_{\pm 0.56}$ & 4.50 & 9.50 & $2.92_{\pm 1.70}$ & $4.17_{\pm 0.38}$ & $0.00_{\pm 0.00}$ \\
GCG & 33.75 & 7.50 & $12.75_{\pm 6.15}$ & 31.25 & 16.25 & $20.00_{\pm 10.00}$ & $14.58_{\pm 2.60}$ & $0.00_{\pm 0.00}$ \\
HumanJb. & 32.75 & 30.00 & $17.65_{\pm 4.81}$ & 32.50 & 23.25 & $27.17_{\pm 5.16}$ & $20.92_{\pm 6.01}$ & $0.00_{\pm 0.00}$ \\
PAP & 6.25 & 3.75 & $3.75_{\pm 0.00}$ & 5.00 & 7.50 & $3.33_{\pm 0.72}$ & $0.42_{\pm 0.72}$ & $0.00_{\pm 0.00}$ \\
PEZ & 8.25 & 2.75 & $5.55_{\pm 0.11}$ & 7.75 & 11.00 & $4.25_{\pm 2.54}$ & $5.33_{\pm 0.29}$ & $0.00_{\pm 0.00}$ \\
UAT & 5.00 & 0.00 & $4.50_{\pm 1.12}$ & 5.00 & 11.25 & $2.92_{\pm 1.91}$ & $2.92_{\pm 1.44}$ & $0.00_{\pm 0.00}$ \\
\midrule
Average & 19.50 & 12.89 & $9.91_{\pm 2.60}$ & 19.00 & 13.61 & $12.30_{\pm 3.19}$ & $8.00_{\pm 0.94}$ & $0.00_{\pm 0.00}$ \\
\bottomrule
\end{tabular*}
\end{small}\end{table*}

\begin{table*}[!t]
\centering
\caption{Per-attack HarmBench ASR for \texttt{Ministral-8B}. Values are percentages; lower is better. \algo{} denotes the default LSE variant. Learning-based methods are reported as mean$_{\pm\mathrm{std}}$ over independent runs, while deterministic baselines use fixed greedy decoding.}
\label{tab:mistral-results-transposed}
\vskip 0.15in
\vspace{-5mm}
\begin{small}
\setlength{\tabcolsep}{0pt}
\begin{tabular*}{\textwidth}{@{\extracolsep{\fill}}lcccccccc}
\toprule
\textbf{Attack Method} & \textbf{Base} & \textbf{Self-Reminder} & \textbf{Circuit Breaker} & \textbf{ActAdd} & \textbf{DirAbl} & \textbf{ReFT-r1} & \textbf{SaP} & \textbf{\algo{}} \\ 
\midrule
AutoDAN & 58.75 & 61.25 & $32.50_{\pm 3.42}$ & 62.50 & 61.25 & $59.58_{\pm 1.44}$ & $35.83_{\pm 2.60}$ & $6.50_{\pm 14.53}$ \\
AutoPrompt & 62.50 & 12.50 & $34.25_{\pm 0.68}$ & 61.25 & 63.75 & $52.08_{\pm 7.53}$ & $28.33_{\pm 9.38}$ & $0.00_{\pm 0.00}$ \\
DirectReq. & 41.25 & 12.50 & $27.25_{\pm 1.37}$ & 43.75 & 43.75 & $36.67_{\pm 0.72}$ & $26.25_{\pm 2.50}$ & $0.00_{\pm 0.00}$ \\
GBDA & 61.25 & 12.25 & $37.95_{\pm 4.52}$ & 60.50 & 60.75 & $50.08_{\pm 3.30}$ & $25.92_{\pm 9.71}$ & $0.00_{\pm 0.00}$ \\
GCG & 62.50 & 25.00 & $39.25_{\pm 0.68}$ & 62.50 & 62.50 & $48.75_{\pm 7.60}$ & $27.50_{\pm 4.51}$ & $0.00_{\pm 0.00}$ \\
HumanJb. & 47.50 & 38.25 & $28.95_{\pm 3.83}$ & 48.00 & 47.25 & $36.08_{\pm 6.00}$ & $23.58_{\pm 1.04}$ & $0.95_{\pm 2.12}$ \\
PAP & 31.25 & 10.00 & $24.00_{\pm 1.37}$ & 31.25 & 31.25 & $19.58_{\pm 2.60}$ & $12.50_{\pm 1.25}$ & $0.00_{\pm 0.00}$ \\
PEZ & 38.25 & 4.25 & $21.35_{\pm 1.92}$ & 37.00 & 39.25 & $20.33_{\pm 1.84}$ & $14.83_{\pm 2.32}$ & $0.05_{\pm 0.11}$ \\
UAT & 57.50 & 6.25 & $35.50_{\pm 0.68}$ & 55.00 & 56.25 & $39.58_{\pm 1.91}$ & $20.42_{\pm 2.60}$ & $0.00_{\pm 0.00}$ \\
\midrule
Average & 51.19 & 20.25 & $30.19_{\pm 2.51}$ & 51.31 & 51.78 & $37.56_{\pm 3.31}$ & $23.91_{\pm 3.49}$ & $0.83_{\pm 1.86}$ \\
\bottomrule
\end{tabular*}
\end{small}
\vspace{-5mm}\end{table*}

\begin{table*}[!t]
\centering
\caption{Per-attack HarmBench ASR for \texttt{Llama-2-7B}. Values are percentages; lower is better. \algo{} denotes the default LSE variant. Learning-based methods are reported as mean$_{\pm\mathrm{std}}$ over independent runs, while deterministic baselines use fixed greedy decoding.}
\label{tab:llama-results-transposed}
\vskip 0.15in
\vspace{-5mm}
\begin{small}
\setlength{\tabcolsep}{0pt}
\begin{tabular*}{\textwidth}{@{\extracolsep{\fill}}lcccccccc}
\toprule
\textbf{Attack Method} & \textbf{Base} & \textbf{Self-Reminder} & \textbf{Circuit Breaker} & \textbf{ActAdd} & \textbf{DirAbl} & \textbf{ReFT-r1} & \textbf{SaP} & \textbf{\algo{}} \\ 
\midrule
AutoDAN & 7.50 & 6.25 & $9.17_{\pm 1.44}$ & 6.25 & 3.75 & $7.50_{\pm 2.17}$ & $9.17_{\pm 0.72}$ & $1.00_{\pm 1.63}$ \\
AutoPrompt & 11.25 & 2.50 & $13.75_{\pm 0.00}$ & 13.75 & 5.00 & $11.25_{\pm 1.25}$ & $15.42_{\pm 1.44}$ & $3.50_{\pm 5.48}$ \\
DirectReq. & 2.50 & 2.50 & $2.92_{\pm 0.72}$ & 2.50 & 0.00 & $1.67_{\pm 0.72}$ & $2.50_{\pm 0.00}$ & $0.50_{\pm 1.12}$ \\
GBDA & 2.00 & 2.00 & $2.42_{\pm 0.29}$ & 2.00 & 1.00 & $2.42_{\pm 0.52}$ & $2.67_{\pm 0.14}$ & $0.55_{\pm 1.10}$ \\
GCG & 22.50 & 7.50 & $23.33_{\pm 0.72}$ & 25.00 & 11.25 & $21.67_{\pm 1.44}$ & $22.50_{\pm 1.25}$ & $4.75_{\pm 8.12}$ \\
HumanJb. & 1.25 & 1.75 & $1.50_{\pm 0.00}$ & 1.50 & 1.00 & $0.83_{\pm 0.38}$ & $1.25_{\pm 0.00}$ & $0.20_{\pm 0.45}$ \\
PAP & 2.50 & 2.50 & $2.08_{\pm 0.72}$ & 2.50 & 0.00 & $2.92_{\pm 1.91}$ & $3.33_{\pm 0.72}$ & $0.25_{\pm 0.56}$ \\
PEZ & 2.50 & 1.50 & $2.67_{\pm 0.14}$ & 2.25 & 3.00 & $1.67_{\pm 0.63}$ & $2.58_{\pm 0.14}$ & $0.50_{\pm 0.71}$ \\
UAT & 2.50 & 2.50 & $2.50_{\pm 0.00}$ & 2.50 & 3.75 & $2.50_{\pm 0.00}$ & $2.50_{\pm 0.00}$ & $0.25_{\pm 0.56}$ \\
\midrule
Average & 6.06 & 3.22 & $6.70_{\pm 0.31}$ & 6.47 & 3.19 & $3.43_{\pm 0.35}$ & $6.88_{\pm 0.21}$ & $1.28_{\pm 2.17}$ \\
\bottomrule
\end{tabular*}
\end{small}
\vspace{-4mm}\end{table*}

\begin{table*}[!t]
\centering
\caption{Per-attack HarmBench ASR for \texttt{Qwen2-1.5B}. Values are percentages; lower is better. \algo{} denotes the default LSE variant. Learning-based methods are reported as mean$_{\pm\mathrm{std}}$ over independent runs, while deterministic baselines use fixed greedy decoding.}
\label{tab:qwen-results-transposed}
\vskip 0.15in
\vspace{-5mm}
\begin{small}
\setlength{\tabcolsep}{0pt}
\begin{tabular*}{\textwidth}{@{\extracolsep{\fill}}lcccccccc}
\toprule
\textbf{Attack Method} & \textbf{Base} & \textbf{Self-Reminder} & \textbf{Circuit Breaker} & \textbf{ActAdd} & \textbf{DirAbl} & \textbf{ReFT-r1} & \textbf{SaP} & \textbf{\algo{}} \\ 
\midrule
AutoDAN & 27.50 & 33.75 & $42.92_{\pm 7.94}$ & 7.50 & 27.50 & $4.58_{\pm 0.72}$ & $4.17_{\pm 4.39}$ & $4.75_{\pm 3.79}$ \\
AutoPrompt & 30.00 & 3.75 & $35.00_{\pm 1.25}$ & 27.50 & 17.50 & $1.25_{\pm 1.25}$ & $8.75_{\pm 3.31}$ & $0.00_{\pm 0.00}$ \\
DirectReq. & 17.50 & 6.25 & $8.75_{\pm 2.17}$ & 6.25 & 11.25 & $8.75_{\pm 2.17}$ & $5.00_{\pm 3.31}$ & $0.00_{\pm 0.00}$ \\
GBDA & 28.25 & 4.25 & $33.50_{\pm 0.43}$ & 26.75 & 25.50 & $3.92_{\pm 0.80}$ & $6.67_{\pm 5.92}$ & $0.00_{\pm 0.00}$ \\
GCG & 47.50 & 18.75 & $50.00_{\pm 3.31}$ & 51.25 & 21.25 & $8.33_{\pm 2.60}$ & $7.08_{\pm 8.04}$ & $0.00_{\pm 0.00}$ \\
HumanJb. & 12.75 & 6.25 & $9.92_{\pm 0.95}$ & 7.25 & 10.50 & $6.42_{\pm 0.88}$ & $3.50_{\pm 1.56}$ & $0.90_{\pm 0.82}$ \\
PAP & 11.25 & 5.00 & $12.50_{\pm 2.17}$ & 7.50 & 6.25 & $3.33_{\pm 1.44}$ & $4.17_{\pm 2.60}$ & $0.00_{\pm 0.00}$ \\
PEZ & 15.25 & 1.25 & $6.92_{\pm 0.52}$ & 12.00 & 13.00 & $1.67_{\pm 0.52}$ & $4.25_{\pm 1.80}$ & $0.00_{\pm 0.00}$ \\
UAT & 25.00 & 2.50 & $15.42_{\pm 1.44}$ & 26.25 & 20.00 & $1.25_{\pm 0.00}$ & $4.58_{\pm 1.91}$ & $0.25_{\pm 0.56}$ \\
\midrule
Average & 23.89 & 9.08 & $23.88_{\pm 1.47}$ & 19.14 & 16.97 & $4.17_{\pm 0.47}$ & $5.35_{\pm 3.44}$ & $0.66_{\pm 0.47}$ \\
\bottomrule
\end{tabular*}
\end{small}
\end{table*}

\subsection{Steering Strength Ablations}
\label{appendix:variant-selection-details}
This section studies how the CBF steering strength $\alpha$ controls the safety--utility trade-off.
Larger $\alpha$ enforces a stronger correction toward the learned safe set, which generally improves jailbreak robustness but can increase over-refusal or reduce downstream utility if the intervention becomes too aggressive.
We report two complementary sweeps: Table~\ref{tab:qwen-top2-alpha-ablation} ablates $\alpha$ for the Top-2 variant on \texttt{Qwen2-1.5B}, while Table~\ref{tab:lse-merge-alpha-sweep} ablates the default LSE variant across model families with $\kappa=100$.

\begin{table*}[!t]
\centering
\caption{\algo{} (LSE) steering-strength ablation across model families. We set $\kappa=100$ and vary $\alpha$. ASR, AA ASR, and XSTest are lower-is-better; MMLU and GSM8K are higher-is-better.}
\label{tab:lse-merge-alpha-sweep}
\vskip 0.15in
\vspace{-5mm}
\begin{small}
\setlength{\tabcolsep}{3pt}
\begin{tabular*}{\textwidth}{@{\extracolsep{\fill}}llccccc}
\toprule
\textbf{Model} & \textbf{$\alpha$} & \textbf{ASR} $\downarrow$ & \textbf{AA ASR} $\downarrow$ & \textbf{XSTest} $\downarrow$ & \textbf{MMLU} $\uparrow$ & \textbf{GSM8K} $\uparrow$ \\
\midrule
\multirow{7}{*}{Gemma-2-9B}
 & 0.001 & $3.12_{\pm 3.42}$ & $66.00_{\pm 0.00}$ & $33.00_{\pm 0.66}$ & $71.87_{\pm 0.01}$ & $73.62_{\pm 0.00}$ \\
 & 0.01 & $0.00_{\pm 0.00}$ & $73.00_{\pm 1.10}$ & $35.00_{\pm 0.66}$ & $71.82_{\pm 0.04}$ & $73.81_{\pm 0.12}$ \\
 & 0.1 & $0.00_{\pm 0.00}$ & $53.00_{\pm 23.00}$ & $32.00_{\pm 0.44}$ & $71.87_{\pm 0.09}$ & $73.65_{\pm 0.04}$ \\
 & 0.2 & $0.00_{\pm 0.00}$ & $64.00_{\pm 6.57}$ & $33.80_{\pm 0.22}$ & $71.25_{\pm 0.55}$ & $73.50_{\pm 0.04}$ \\
 & 0.3 & $0.00_{\pm 0.00}$ & $64.00_{\pm 8.76}$ & $33.80_{\pm 1.10}$ & $71.70_{\pm 0.16}$ & $72.93_{\pm 0.75}$ \\
 & 0.5 & $0.00_{\pm 0.00}$ & $69.00_{\pm 5.48}$ & $34.00_{\pm 3.07}$ & $68.58_{\pm 3.19}$ & $71.65_{\pm 1.91}$ \\
 & 1.0 & $0.00_{\pm 0.00}$ & $67.00_{\pm 1.10}$ & $34.20_{\pm 1.53}$ & $65.28_{\pm 6.62}$ & $67.21_{\pm 6.93}$ \\
\midrule
\multirow{7}{*}{Qwen2-1.5B}
 & 0.001 & $0.20_{\pm 0.26}$ & $56.80_{\pm 9.34}$ & $28.00_{\pm 1.70}$ & $55.72_{\pm 0.04}$ & $13.03_{\pm 0.14}$ \\
 & 0.01 & $0.66_{\pm 0.47}$ & $24.80_{\pm 1.79}$ & $29.44_{\pm 1.49}$ & $55.65_{\pm 0.05}$ & $12.95_{\pm 0.63}$ \\
 & 0.1 & $0.06_{\pm 0.10}$ & $12.67_{\pm 13.61}$ & $60.13_{\pm 21.70}$ & $50.24_{\pm 4.58}$ & $12.36_{\pm 1.71}$ \\
 & 0.2 & $0.02_{\pm 0.03}$ & $4.67_{\pm 5.03}$ & $66.53_{\pm 14.34}$ & $45.94_{\pm 7.02}$ & $11.37_{\pm 1.93}$ \\
 & 0.3 & $0.04_{\pm 0.07}$ & $1.33_{\pm 2.31}$ & $69.73_{\pm 8.95}$ & $43.59_{\pm 6.07}$ & $10.69_{\pm 0.72}$ \\
 & 0.5 & $0.00_{\pm 0.00}$ & $9.33_{\pm 11.02}$ & $74.13_{\pm 2.81}$ & $40.45_{\pm 4.03}$ & $9.30_{\pm 1.25}$ \\
 & 1.0 & $0.00_{\pm 0.00}$ & $2.67_{\pm 1.15}$ & $76.93_{\pm 4.64}$ & $37.80_{\pm 2.81}$ & $7.25_{\pm 0.16}$ \\
\midrule
\multirow{7}{*}{Ministral-8B}
 & 0.001 & $18.15_{\pm 25.05}$ & $74.40_{\pm 4.34}$ & $15.76_{\pm 0.92}$ & $64.59_{\pm 0.05}$ & $64.58_{\pm 0.67}$ \\
 & 0.01 & $8.57_{\pm 14.85}$ & $61.33_{\pm 11.02}$ & $15.47_{\pm 6.89}$ & $64.48_{\pm 0.09}$ & $53.02_{\pm 14.03}$ \\
 & 0.1 & $0.83_{\pm 1.86}$ & $19.60_{\pm 2.61}$ & $20.96_{\pm 7.04}$ & $63.07_{\pm 1.95}$ & $62.55_{\pm 2.46}$ \\
 & 0.2 & $0.32_{\pm 0.55}$ & $11.33_{\pm 16.29}$ & $68.27_{\pm 23.81}$ & $55.83_{\pm 11.41}$ & $10.29_{\pm 14.54}$ \\
 & 0.3 & $0.12_{\pm 0.21}$ & $6.00_{\pm 10.39}$ & $65.07_{\pm 25.59}$ & $54.48_{\pm 12.81}$ & $7.28_{\pm 11.01}$ \\
 & 0.5 & $0.02_{\pm 0.03}$ & $3.33_{\pm 4.16}$ & $65.47_{\pm 28.49}$ & $48.69_{\pm 11.62}$ & $6.90_{\pm 10.92}$ \\
 & 1.0 & $0.00_{\pm 0.00}$ & $7.33_{\pm 9.45}$ & $57.87_{\pm 33.11}$ & $43.89_{\pm 13.77}$ & $5.79_{\pm 8.68}$ \\
\midrule
\multirow{7}{*}{Llama-2-7B}
 & 0.001 & $0.93_{\pm 1.58}$ & $18.40_{\pm 11.70}$ & $42.32_{\pm 0.52}$ & $46.55_{\pm 0.01}$ & $21.12_{\pm 0.15}$ \\
 & 0.01 & $1.25_{\pm 1.91}$ & $18.67_{\pm 14.47}$ & $43.60_{\pm 0.80}$ & $46.54_{\pm 0.01}$ & $21.43_{\pm 0.64}$ \\
 & 0.1 & $1.28_{\pm 2.17}$ & $1.60_{\pm 0.89}$ & $43.36_{\pm 0.83}$ & $46.56_{\pm 0.04}$ & $20.94_{\pm 0.62}$ \\
 & 0.2 & $0.93_{\pm 1.32}$ & $2.00_{\pm 2.00}$ & $44.80_{\pm 2.40}$ & $46.12_{\pm 0.67}$ & $18.04_{\pm 4.80}$ \\
 & 0.3 & $0.69_{\pm 1.05}$ & $3.33_{\pm 3.06}$ & $44.13_{\pm 1.97}$ & $45.79_{\pm 1.32}$ & $17.87_{\pm 4.62}$ \\
 & 0.5 & $0.38_{\pm 0.60}$ & $6.00_{\pm 2.00}$ & $45.47_{\pm 1.51}$ & $45.14_{\pm 2.28}$ & $16.20_{\pm 4.71}$ \\
 & 1.0 & $0.30_{\pm 0.52}$ & $3.33_{\pm 4.16}$ & $49.07_{\pm 1.29}$ & $44.36_{\pm 2.89}$ & $16.45_{\pm 4.08}$ \\
\bottomrule
\end{tabular*}
\end{small}
\vspace{-3mm}

\end{table*}

\begin{table}[!t]
\centering
\caption{Ablation study on the effect of steering strength $\alpha$ for \algo{} (Top-2) on Qwen2-1.5B. ASR, AA ASR, and XSTest are lower-is-better; MMLU and GSM8K are higher-is-better.}
\label{tab:qwen-top2-alpha-ablation}
\begin{small}
\setlength{\tabcolsep}{4pt}
\begin{tabular}{lccccc}
\toprule
\textbf{$\alpha$} & \textbf{ASR} $\downarrow$ & \textbf{AA ASR} $\downarrow$ & \textbf{XSTest} $\downarrow$ & \textbf{MMLU} $\uparrow$ & \textbf{GSM8K} $\uparrow$ \\
\midrule
0.01 & $5.05_{\pm 7.95}$ & $31.60_{\pm 4.98}$ & $30.32_{\pm 2.97}$ & $\mathbf{55.69_{\pm 0.03}}$ & $\mathbf{12.43_{\pm 0.79}}$ \\
0.1 & $2.44_{\pm 4.09}$ & $9.60_{\pm 10.43}$ & $49.60_{\pm 23.03}$ & $55.24_{\pm 0.38}$ & $11.65_{\pm 1.69}$ \\
0.2 & $1.53_{\pm 2.65}$ & $3.20_{\pm 4.60}$ & $58.00_{\pm 15.49}$ & $54.83_{\pm 0.42}$ & $10.46_{\pm 2.58}$ \\
0.3 & $0.70_{\pm 1.22}$ & $0.80_{\pm 1.10}$ & $70.24_{\pm 4.68}$ & $54.61_{\pm 0.36}$ & $9.73_{\pm 1.62}$ \\
0.5 & $0.14_{\pm 0.24}$ & $4.40_{\pm 8.76}$ & $76.00_{\pm 0.69}$ & $54.42_{\pm 0.24}$ & $9.25_{\pm 2.27}$ \\
1.0 & $0.00_{\pm 0.00}$ & $3.20_{\pm 1.79}$ & $74.32_{\pm 3.18}$ & $54.27_{\pm 0.23}$ & $7.99_{\pm 1.61}$ \\
\bottomrule
\end{tabular}
\end{small}
\end{table}

Table~\ref{tab:qwen-top2-alpha-ablation} shows the expected monotonic safety--strength pattern for Top-2: increasing $\alpha$ steadily reduces static ASR, while utility decreases gradually and over-refusal increases as the intervention becomes stronger.
This confirms that $\alpha$ provides a direct knob for choosing the desired operating point rather than changing the learned barrier itself.
The adaptive-attack results are less strictly monotonic because the attack search and response generation introduce additional variance, but the best low-ASR settings still occur at moderate-to-large steering strengths.

Table~\ref{tab:lse-merge-alpha-sweep} shows a similar safety--utility trade-off for the LSE variant across model families.
Small $\alpha$ values preserve utility closest to the base model but may leave adaptive attacks insufficiently suppressed, whereas larger $\alpha$ values provide stronger defense at the cost of higher over-refusal and lower utility on some models.
This motivates selecting $\alpha$ per deployment requirement: conservative settings are appropriate when preserving benign capability is most important, while stronger settings are appropriate when adversarial robustness is prioritized.

\subsection{Extended Utility Results}
\label{appendix:extended-utility-results}
In addition to MMLU and GSM8K in the main tables, we evaluate Gemma-2-9B on broader utility benchmarks in Table~\ref{tab:gemma-extended-utility}.
These results test whether the safety interventions preserve instruction following, factuality, conversational quality, and difficult question answering beyond the two primary utility metrics.
\begin{table*}[!t]
\centering
\caption{Extended utility benchmark results for Gemma-2-9B. Higher is better for all metrics.}
\label{tab:gemma-extended-utility}
\vskip 0.15in
\vspace{-5mm}
\begin{small}
\setlength{\tabcolsep}{1pt}
\begin{tabular*}{\textwidth}{@{\extracolsep{\fill}}lcccccccc}
\toprule
\textbf{Metric} & \textbf{Base} & \textbf{Self-Rem.} 
& \textbf{Cir. Breaker} 
& \textbf{ActAdd} & \textbf{DirAbl} & \textbf{ReFT-r1} & \textbf{SaP} & \textbf{\algo{}} \\
\midrule
IFEval $\uparrow$ & 65.80 & 75.30 & $62.03_{\pm 10.84}$ & 66.73 & 33.09 & $74.90_{\pm 0.14}$ & $66.23_{\pm 0.39}$ & $64.85_{\pm 0.43}$ \\
TruthfulQA $\uparrow$ & 60.14 & 61.88 & $61.03_{\pm 1.38}$ & 60.15 & 55.35 & $58.68_{\pm 0.86}$ & $60.14_{\pm 0.02}$ & $60.57_{\pm 0.48}$ \\
MT-Bench $\uparrow$ & 7.19 & 7.14 & $5.91_{\pm 2.15}$ & 7.17 & 4.96 & $7.19_{\pm 0.08}$ & $7.16_{\pm 0.05}$ & $7.15_{\pm 0.03}$ \\
GPQA Diamond $\uparrow$ & 38.38 & 35.35 & $38.04_{\pm 0.58}$ & 38.38 & 37.88 & $34.85_{\pm 3.50}$ & $38.38_{\pm 0.00}$ & $38.18_{\pm 0.45}$ \\
\bottomrule
\end{tabular*}
\end{small}
\vspace{-3mm}
\end{table*}
Overall, \algo{} remains close to the base model on the extended utility suite: it preserves utility across MT-Bench~\citep{zheng2023judging}, TruthfulQA~\citep{lin2022truthfulqa} GPQA Diamond~\citep{rein2023gpqagraduatelevelgoogleproofqa}, and IFEval~\citep{zhou2023instruction} and avoids the severe degradation observed for stronger fixed-direction interventions such as directional ablation.
This supports the conclusion that classifier-guided CBF steering can improve safety while maintaining broad model utility.

\subsection{Computational Benchmarking}
\label{sec:computational-benchmark-details}
This section provides additional details for the main-paper steering-latency results in Table~\ref{tab:latency} and reports full generation cost in Table~\ref{tab:qwen-generation-efficiency-variants}.
For Table~\ref{tab:latency}, we measure inference-time overhead using hidden-state activations extracted from the adversarial unsafe dataset.
All per-token latency experiments are run on a single NVIDIA H200 GPU with batch size $1$, matching the real-time decoding setting where a controller must act at each generated token.
We compare \texttt{SaP} with $K=30$ polytope facets against \algo{} with neural CBF heads using comparable parameter counts.
Following the \texttt{SaP} implementation~\citep{ICML2025_chenlearning}, the baseline uses $100$ gradient-based optimization steps.

\para{Full generation cost.}
To complement the per-token steering latency in Table~\ref{tab:latency}, we also measure end-to-end completion-generation time on the HarmBench DirectRequest split.
This full-generation benchmark uses the same single NVIDIA H200 GPU and batch size $1$ setup, with identical prompts, greedy decoding, random seed, and maximum-new-token budget across methods.

\begin{table}[!t]
\centering
\caption{HarmBench generation time on \texttt{Qwen2-1.5B}. We time only the completion-generation step for the DirectRequest split, using identical prompts, seed, greedy decoding, batch size, and maximum-new-token budget across methods.}
\label{tab:qwen-generation-efficiency-variants}
\begin{small}
\setlength{\tabcolsep}{10pt}
\begin{tabular}{lc}
\toprule
\textbf{Method} & \textbf{Generation Time (s)} $\downarrow$ \\
\midrule
Base model & \textbf{35.27} \\
SaP & 102.05 \\
\algo{} (LSE) & 39.44 \\
\bottomrule
\end{tabular}
\end{small}
\vspace{-4mm}
\end{table}

Table~\ref{tab:qwen-generation-efficiency-variants} shows that the default LSE variant adds only modest overhead relative to the unmodified model, while \texttt{SaP} is substantially slower because it solves an optimization problem during decoding.
This motivates using LSE as the default \algo{} variant in the main experiments: it composes multiple learned safety constraints while preserving near-base generation cost.

\para{Neural CBF sizes ablations.}
We next isolate two sources of computational scaling: the neural CBF architecture size and the number of barriers $K$.
Table~\ref{tab:phase-scaling-params} varies the number of trainable neural-CBF parameters while keeping $K$ fixed, and Table~\ref{tab:phase-scaling-k} varies $K$ while keeping the per-barrier architecture fixed.
These measurements report steering overhead per generated token and therefore complement the end-to-end generation-time comparison above.

\begin{table*}[!t]
  \centering
  \captionsetup{justification=raggedright,singlelinecheck=false}
  \caption{Computational ablation over neural CBF size. We vary the number of trainable barrier parameters while keeping the number of CBF heads fixed. Values are per-token steering latency in milliseconds, reported as mean$\pm$std. The largest \algo{} setting has 23.64M parameters, comparable to the 25.67M-parameter SaP safety module. For \algo{} (LSE), barrier checking and gradient computation are fused into an operation and the cost appears under ``Check CBFs.''}
  \label{tab:phase-scaling-params}
  \begin{small}
  \setlength{\tabcolsep}{0pt}
  \begin{tabular*}{\textwidth}{@{\extracolsep{\fill}}l r c c c c c}
  \toprule
  \textbf{Method} & \textbf{\#Params.} & \textbf{Check CBFs} & \textbf{Select CBFs} & \textbf{Gradient Calc.} & \textbf{Compute Steering} & \textbf{Total} \\
  \midrule
  \multirow{6}{*}{\algo{} (Top-2)}
  & 0.8M    & $0.36_{\pm 0.04}$ & $0.06_{\pm 0.01}$ & $1.24_{\pm 0.11}$ & $5.15_{\pm 10.42}$ & $6.82_{\pm 10.42}$ \\
  & 6.3M  & $0.37_{\pm 0.01}$ & $0.06_{\pm 0.00}$ & $1.24_{\pm 0.05}$ & $3.53_{\pm 8.50}$  & $5.21_{\pm 8.50}$ \\
  & 12.6M & $0.41_{\pm 0.04}$ & $0.06_{\pm 0.01}$ & $1.31_{\pm 0.13}$ & $3.52_{\pm 8.35}$  & $5.31_{\pm 8.35}$ \\
  & 15.7M & $0.82_{\pm 0.01}$ & $0.06_{\pm 0.00}$ & $2.13_{\pm 0.06}$ & $28.78_{\pm 6.98}$ & $31.79_{\pm 6.98}$ \\
  & 19.5M & $0.85_{\pm 0.01}$ & $0.06_{\pm 0.00}$ & $2.12_{\pm 0.06}$ & $24.23_{\pm 10.62}$ & $27.26_{\pm 10.62}$ \\
  & 23.6M & $0.87_{\pm 0.10}$ & $0.06_{\pm 0.01}$ & $2.17_{\pm 0.21}$ & $26.34_{\pm 10.99}$ & $29.45_{\pm 11.00}$ \\
  \midrule
  \multirow{6}{*}{\algo{} (QP)}
  & 0.8M    & $0.37_{\pm 0.01}$ & $0.06_{\pm 0.00}$ & $2.45_{\pm 0.13}$ & $22.53_{\pm 12.66}$ & $25.41_{\pm 12.66}$ \\
  & 6.3M  & $0.38_{\pm 0.05}$ & $0.06_{\pm 0.02}$ & $2.47_{\pm 0.22}$ & $14.30_{\pm 14.72}$ & $17.20_{\pm 14.73}$ \\
  & 12.6M & $0.41_{\pm 0.05}$ & $0.06_{\pm 0.01}$ & $2.51_{\pm 0.21}$ & $7.47_{\pm 12.30}$  & $10.44_{\pm 12.30}$ \\
  & 15.7M & $0.82_{\pm 0.09}$ & $0.06_{\pm 0.01}$ & $5.07_{\pm 0.41}$ & $28.01_{\pm 4.38}$  & $33.95_{\pm 4.44}$ \\
  & 19.5M & $0.85_{\pm 0.03}$ & $0.06_{\pm 0.00}$ & $5.06_{\pm 0.07}$ & $27.35_{\pm 5.66}$  & $33.32_{\pm 5.66}$ \\
  & 23.6M & $0.84_{\pm 0.09}$ & $0.06_{\pm 0.01}$ & $5.12_{\pm 0.43}$ & $27.75_{\pm 5.82}$  & $33.77_{\pm 5.86}$ \\
  \midrule
  \multirow{6}{*}{\algo{} (LSE)}
  & 0.8M    & $1.09_{\pm 0.05}$ & $0.04_{\pm 0.00}$ & $0.01_{\pm 0.00}$ & $0.13_{\pm 0.00}$ & $1.28_{\pm 0.05}$ \\
  & 6.3M  & $1.09_{\pm 0.02}$ & $0.04_{\pm 0.00}$ & $0.01_{\pm 0.00}$ & $0.13_{\pm 0.00}$ & $1.28_{\pm 0.02}$ \\
  & 12.6M & $1.15_{\pm 0.08}$ & $0.04_{\pm 0.00}$ & $0.01_{\pm 0.00}$ & $0.13_{\pm 0.00}$ & $1.34_{\pm 0.08}$ \\
  & 15.7M & $2.15_{\pm 0.04}$ & $0.04_{\pm 0.00}$ & $0.01_{\pm 0.00}$ & $0.14_{\pm 0.03}$ & $2.34_{\pm 0.06}$ \\
  & 19.5M & $2.19_{\pm 0.07}$ & $0.04_{\pm 0.00}$ & $0.01_{\pm 0.00}$ & $0.13_{\pm 0.00}$ & $2.38_{\pm 0.07}$ \\
  & 23.6M & $2.21_{\pm 0.04}$ & $0.05_{\pm 0.00}$ & $0.01_{\pm 0.00}$ & $0.14_{\pm 0.04}$ & $2.40_{\pm 0.06}$ \\
  \bottomrule
  \end{tabular*}
  \end{small}
  \vspace{-4mm}
  \end{table*}

Table~\ref{tab:phase-scaling-params} shows that the LSE update remains lightweight as the neural CBF size increases.
For LSE, CBF checking and merged-barrier gradient computation are fused into a single autograd pass; therefore, its cost appears primarily in the ``Check CBFs'' column, while the separate ``Gradient Calc.'' column is near zero by construction.
By contrast, QP requires explicit gradient and solve steps whose runtime is more variable, making LSE the preferred default when both speed and modular composition are important.

\para{Ablation on number of constraints $K$.}
We next vary the number of learned CBF heads while keeping the per-head architecture fixed.
This experiment isolates how each merge rule scales as more safety constraints are checked and composed during decoding.

  \begin{table*}[!t]
  \centering
  \captionsetup{justification=raggedright,singlelinecheck=false}
  \caption{Computational ablation over the number of CBF $K$. We vary $K$ while keeping the per-head neural architecture fixed. Values are per-token steering latency in milliseconds, reported as mean$_{\pm\mathrm{std}}$.}
  \label{tab:phase-scaling-k}
  \begin{small}
  \setlength{\tabcolsep}{0pt}
  \begin{tabular*}{\textwidth}{@{\extracolsep{\fill}}l c c c c c c}
  \toprule
  \textbf{Method} & \textbf{$K$} & \textbf{Check CBFs} & \textbf{Select CBFs} & \textbf{Gradient Calc.} & \textbf{Compute Steering} & \textbf{Total} \\
  \midrule
  \multirow{8}{*}{\algo{} (Top-2)}
  & 4  & $0.37_{\pm 0.04}$ & $0.06_{\pm 0.01}$ & $1.24_{\pm 0.12}$ & $5.12_{\pm 10.32}$ & $6.78_{\pm 10.33}$ \\
  & 12 & $0.84_{\pm 0.04}$ & $0.06_{\pm 0.01}$ & $1.26_{\pm 0.07}$ & $14.19_{\pm 14.11}$ & $16.35_{\pm 14.11}$ \\
  & 20 & $1.32_{\pm 0.02}$ & $0.06_{\pm 0.00}$ & $1.26_{\pm 0.04}$ & $20.32_{\pm 13.90}$ & $22.97_{\pm 13.90}$ \\
  & 28 & $1.79_{\pm 0.04}$ & $0.06_{\pm 0.00}$ & $1.26_{\pm 0.05}$ & $6.95_{\pm 11.90}$ & $10.07_{\pm 11.90}$ \\
  & 36 & $2.29_{\pm 0.20}$ & $0.06_{\pm 0.01}$ & $1.27_{\pm 0.12}$ & $12.82_{\pm 14.59}$ & $16.44_{\pm 14.60}$ \\
  & 44 & $2.75_{\pm 0.03}$ & $0.06_{\pm 0.00}$ & $1.26_{\pm 0.02}$ & $26.02_{\pm 10.76}$ & $30.09_{\pm 10.76}$ \\
  & 52 & $3.15_{\pm 0.03}$ & $0.06_{\pm 0.00}$ & $1.24_{\pm 0.03}$ & $26.25_{\pm 8.64}$ & $30.69_{\pm 8.64}$ \\
  & 60 & $3.65_{\pm 0.03}$ & $0.06_{\pm 0.00}$ & $1.26_{\pm 0.04}$ & $24.23_{\pm 11.24}$ & $29.19_{\pm 11.24}$ \\
  \midrule
  \multirow{8}{*}{\algo{} (QP)}
  & 4  & $0.37_{\pm 0.01}$ & $0.06_{\pm 0.00}$ & $2.42_{\pm 0.04}$ & $22.24_{\pm 12.49}$ & $25.09_{\pm 12.49}$ \\
  & 12 & $0.86_{\pm 0.09}$ & $0.06_{\pm 0.01}$ & $12.28_{\pm 0.86}$ & $21.93_{\pm 13.22}$ & $35.13_{\pm 13.26}$ \\
  & 20 & $1.30_{\pm 0.11}$ & $0.06_{\pm 0.01}$ & $29.77_{\pm 1.54}$ & $24.32_{\pm 10.47}$ & $55.45_{\pm 10.60}$ \\
  & 28 & $1.76_{\pm 0.03}$ & $0.06_{\pm 0.00}$ & $55.31_{\pm 0.33}$ & $18.06_{\pm 13.50}$ & $75.19_{\pm 13.50}$ \\
  & 36 & $2.28_{\pm 0.21}$ & $0.05_{\pm 0.00}$ & $90.99_{\pm 2.69}$ & $24.95_{\pm 11.30}$ & $118.28_{\pm 11.67}$ \\
  & 44 & $2.72_{\pm 0.25}$ & $0.05_{\pm 0.01}$ & $132.04_{\pm 3.22}$ & $29.31_{\pm 1.53}$ & $164.13_{\pm 3.58}$ \\
  & 52 & $3.14_{\pm 0.03}$ & $0.05_{\pm 0.00}$ & $181.78_{\pm 0.72}$ & $29.51_{\pm 0.33}$ & $214.48_{\pm 0.83}$ \\
  & 60 & $3.64_{\pm 0.34}$ & $0.05_{\pm 0.01}$ & $241.60_{\pm 4.40}$ & $29.49_{\pm 1.40}$ & $274.77_{\pm 4.71}$ \\
  \midrule
  \multirow{8}{*}{\algo{} (LSE)}
  & 4  & $1.08_{\pm 0.02}$ & $0.04_{\pm 0.00}$ & $0.01_{\pm 0.00}$ & $0.13_{\pm 0.00}$ & $1.27_{\pm 0.02}$ \\
  & 12 & $2.22_{\pm 0.03}$ & $0.05_{\pm 0.00}$ & $0.01_{\pm 0.00}$ & $0.14_{\pm 0.03}$ & $2.42_{\pm 0.04}$ \\
  & 20 & $3.33_{\pm 0.06}$ & $0.05_{\pm 0.00}$ & $0.01_{\pm 0.00}$ & $0.14_{\pm 0.00}$ & $3.53_{\pm 0.06}$ \\
  & 28 & $4.46_{\pm 0.21}$ & $0.05_{\pm 0.00}$ & $0.01_{\pm 0.03}$ & $0.14_{\pm 0.04}$ & $4.66_{\pm 0.22}$ \\
  & 36 & $5.66_{\pm 0.24}$ & $0.05_{\pm 0.00}$ & $0.01_{\pm 0.00}$ & $0.14_{\pm 0.01}$ & $5.86_{\pm 0.24}$ \\
  & 44 & $6.61_{\pm 0.12}$ & $0.05_{\pm 0.00}$ & $0.01_{\pm 0.00}$ & $0.14_{\pm 0.00}$ & $6.81_{\pm 0.13}$ \\
  & 52 & $7.65_{\pm 0.10}$ & $0.05_{\pm 0.00}$ & $0.01_{\pm 0.00}$ & $0.14_{\pm 0.00}$ & $7.86_{\pm 0.11}$ \\
  & 60 & $8.63_{\pm 0.32}$ & $0.06_{\pm 0.01}$ & $0.01_{\pm 0.00}$ & $0.14_{\pm 0.01}$ & $8.84_{\pm 0.33}$ \\
  \bottomrule
  \end{tabular*}
  \end{small}
  \vspace{-4mm}
  \end{table*}

Table~\ref{tab:phase-scaling-k} shows the expected scaling pattern with the number of CBF heads.
QP becomes expensive as $K$ grows because it differentiates and solves over all active constraints, while LSE scales much more smoothly by merging the barriers into a single differentiable surrogate before computing the steering update.
This computational behavior is why we recommend LSE for the default multi-constraint setting and reserve Top-2 primarily as a simple closed-form ablation.

\subsection{Modular Composition of Safety Barriers}
\label{appendix:compose-cbf-exp}
For modular composition, we use WildGuardMix~\citep{han2024wildguardopenonestopmoderation}, which covers 14 safety categories.
For each category, we sample 400 behaviors and split them into 320 training behaviors and 80 held-out evaluation behaviors.
At inference time, this yields $1{,}120$ evaluation samples ($80\times14$).
We report the harmfulness rate measured by \texttt{HarmBench-Llama-2-13b-cls}~\citep{mazeika2024harmbench}.

Each category-specific safety constraint uses the neural CBF architecture in Appendix~\ref{appendix:barrier-implementation-details}.
To isolate composition, we train one CBF per semantic category and use a fixed steering strength $\alpha=1$ for all methods.
Table~\ref{tab:composition-results-detailed} reports category-level individual-barrier results, alternative Top-2/QP/LSE merge rules, and a linear-probe ablation that removes the neural hidden layer from each category barrier.

\begin{table*}[!t]
\centering
\caption{Gemma-2-9B LSE smoothing sensitivity for \algo{} with fixed $\alpha=0.2$. We report all metrics mean$_{\pm\mathrm{std}}$ over 5 runs.}
\label{tab:lse-kappa-sensitivity}
\vskip 0.15in
\vspace{-5mm}
\begin{small}
\setlength{\tabcolsep}{5pt}
\begin{tabular*}{\textwidth}{@{\extracolsep{\fill}}lccccc}
\toprule
\textbf{$\kappa$} & \textbf{ASR} $\downarrow$ & \textbf{AA ASR} $\downarrow$ & \textbf{XSTest} $\downarrow$ & \textbf{MMLU} $\uparrow$ & \textbf{GSM8K} $\uparrow$ \\
\midrule
1 & $0.00_{\pm 0.00}$ & $54.00_{\pm 19.70}$ & $36.80_{\pm 9.06}$ & $44.28_{\pm 9.97}$ & $54.69_{\pm 4.60}$ \\
10 & $0.00_{\pm 0.00}$ & $58.00_{\pm 19.70}$ & $37.87_{\pm 7.27}$ & $68.80_{\pm 1.97}$ & $70.00_{\pm 5.42}$ \\
25 & $0.00_{\pm 0.00}$ & $54.67_{\pm 20.03}$ & $40.00_{\pm 10.83}$ & $70.45_{\pm 1.26}$ & $68.16_{\pm 10.24}$ \\
50 & $0.00_{\pm 0.00}$ & $56.00_{\pm 19.70}$ & $41.20_{\pm 11.84}$ & $70.43_{\pm 1.16}$ & $67.90_{\pm 10.10}$ \\
100 & $0.00_{\pm 0.00}$ & $54.00_{\pm 18.33}$ & $38.53_{\pm 8.20}$ & $71.04_{\pm 0.63}$ & $67.42_{\pm 10.53}$ \\
\bottomrule
\end{tabular*}
\end{small}
\end{table*}

\begin{table*}[!t]
\centering
\caption{Detailed WildGuard modular-composition results for \algo{} on \texttt{Qwen2-1.5B}. This appendix table expands Table~\ref{tab:composition_results} by showing category-level individual CBF results, alternative Top-2/QP/LSE merging rules, and a linear-probe ablation that replaces each neural CBF with a linear barrier. The LSE rows use the selected main-table smoothing value $\kappa=0.01$.}
\label{tab:composition-results-detailed}
\vskip 0.15in
\vspace{-5mm}
\begin{small}
\setlength{\tabcolsep}{5pt}
\begin{tabular}{llc}
\toprule
\textbf{Barrier Type} & \textbf{Category / Merge Rule} & \textbf{Harmful Rate $\downarrow$} \\
\midrule
Original Model & -- & 22.77 \\
\midrule
\multirow{14}{*}{{Single \algo{}}}
& Copyright & $5.71_{\pm 0.00}$ \\
& Cyberattack & $3.12_{\pm 0.00}$ \\
& Defamation & $2.95_{\pm 0.00}$ \\
& Discrimination & $3.81_{\pm 1.91}$ \\
& Disinformation & $5.27_{\pm 0.00}$ \\
& Fraud & $3.66_{\pm 0.00}$ \\
& Harm Dissemination & $5.27_{\pm 0.00}$ \\
& Mental Health & $4.46_{\pm 0.00}$ \\
& Others & $5.30_{\pm 0.36}$ \\
& Private Info & $4.82_{\pm 0.00}$ \\
& Sensitive Info & $6.22_{\pm 1.13}$ \\
& Sexual & $5.09_{\pm 0.00}$ \\
& Toxic & $6.61_{\pm 0.00}$ \\
& Violence & $4.02_{\pm 0.93}$ \\
\cmidrule(lr){2-3}
& Average & $4.74_{\pm 1.21}$ \\
\midrule
\multirow{3}{*}{\algo{} (5-layer MLP)}
& Top-2 & $10.60_{\pm 1.37}$ \\
& QP & $1.82_{\pm 0.29}$ \\
& LSE ($\kappa=0.01$) & $\mathbf{0.86_{\pm 0.04}}$ \\
\midrule
\multirow{3}{*}{\algo{} (Linear Probe)}
& Top-2 & $12.44_{\pm 0.95}$ \\
& QP & $8.87_{\pm 0.15}$ \\
& LSE ($\kappa=0.01$) & $2.89_{\pm 0.69}$ \\
\bottomrule
\end{tabular}
\end{small}
\end{table*}

\begin{table*}[!t]
\centering
\caption{WildGuard LSE smoothing sensitivity for composed \algo{} barriers on \texttt{Qwen2-1.5B}. We compose 14 category-specific barriers and evaluate $1{,}120$ held-out WildGuardMix samples per seed. Harmful rate is reported as mean$_{\pm\mathrm{std}}$; lower is better. The main WildGuard Table~\ref{tab:composition_results} uses $\kappa=0.01$.}
\label{tab:wildguard-lse-kappa-sensitivity}
\vskip 0.15in
\vspace{-5mm}
\begin{small}
\setlength{\tabcolsep}{8pt}
\begin{tabular*}{\textwidth}{@{\extracolsep{\fill}}lclc}
\toprule
\textbf{$\kappa$} & \textbf{Harmful Rate $\downarrow$} & \textbf{$\kappa$} & \textbf{Harmful Rate $\downarrow$} \\
\midrule
0.01 & $\mathbf{0.86_{\pm 0.04}}$ & 1.0 & $4.26_{\pm 0.68}$ \\
0.03 & $0.71_{\pm 0.07}$ & 2.0 & $7.71_{\pm 0.65}$ \\
0.05 & $1.16_{\pm 0.13}$ & 5.0 & $9.02_{\pm 0.15}$ \\
0.1 & $0.89_{\pm 0.32}$ & 10.0 & $8.48_{\pm 0.39}$ \\
0.2 & $1.16_{\pm 0.26}$ & 20.0 & $7.77_{\pm 1.17}$ \\
0.5 & $1.61_{\pm 0.51}$ & 50.0 & $7.59_{\pm 1.79}$ \\
\bottomrule
\end{tabular*}
\end{small}
\vspace{-5mm}
\end{table*}

\subsection{Ablation on LSE Smoothing Parameter \texorpdfstring{$\kappa$}{kappa}}
\label{appendix:lse-kappa-ablation}
The LSE merge rule uses $\kappa$ to control how sharply the smooth minimum in~\cref{eqn:compose} approximates the most violated barrier constraint.
Because the exponential weights are proportional to $e^{-\kappa(b_k(h)-\delta)}$, smaller $\kappa$ gives a smoother merge that spreads the steering signal across more constraints, while larger $\kappa$ concentrates the merge on the most violated constraints and approaches the hard-minimum behavior used by Top-2-style selection.
As a result, $\kappa$ becomes more important when many constraints are composed: with only a few barriers, the smooth minimum is relatively stable, but with many independently trained barriers, the choice of $\kappa$ controls how broadly the controller considers all constraints versus focusing on the currently most active ones.

Table~\ref{tab:lse-kappa-sensitivity} ablates this parameter for HarmBench on \texttt{Gemma-2-9B} with the LSE variant and fixed steering strength $\alpha=0.2$.
Across all tested values, static HarmBench ASR remains at $0.00\%$, indicating that the LSE merge is not highly sensitive to $\kappa$ in this non-adaptive HarmBench setting once the steering strength is fixed.
The main difference appears in utility: very small $\kappa$ over-smooths the constraints and can unnecessarily distort representations, while larger values better preserve MMLU/GSM8K performance.
This supports using a sharper LSE approximation for the HarmBench experiments, where we set $\kappa=100$ by default.
For the WildGuard modular-composition experiment, Table~\ref{tab:wildguard-lse-kappa-sensitivity} shows that $\kappa$ is more sensitive when composing 14 independently trained category barriers.
In this setting, a smaller $\kappa$ is preferable because it considers a broader set of category constraints rather than approximating a hard minimum too aggressively, so we use $\kappa=0.01$ for the WildGuard LSE composition results.

\subsection{Ablation on Neural CBF Architecture}
\label{appendix:neural-cbf-ablation}
To test whether the neural CBF architecture is necessary, we compare the default nonlinear category barriers against a linear-probe variant in the WildGuard modular-composition experiment.
The linear-probe ablation retrains all 14 category-specific barriers without the neural intermediate layer, while keeping the same data split, composition rules, and evaluation protocol.
Concretely, each linear-probe barrier is a single affine classifier $b_k(h)=w_k^	op h+c_k$ applied directly to the hidden state, so the safe set for each category is a halfspace rather than the nonlinear decision boundary induced by the default 5-layer MLP.
As shown in Table~\ref{tab:composition-results-detailed}, the neural barriers consistently outperform the linear-probe variants under the same composition rules: for example, LSE composition improves from $2.89_{\pm 0.69}$ harmful rate with linear probes to $0.86_{\pm 0.04}$ with neural CBFs.
The gap is larger for QP and Top-2 composition, suggesting that nonlinear CBF boundaries provide a more faithful representation of category-specific safety constraints before the barriers are merged.
These results support the design choice of learning nonlinear latent CBFs rather than relying only on linear hidden-state classifiers for modular safety composition.

\subsection{Conditional Fixed-Direction Steering Baselines}
\label{appendix:conditional-fixed-steering}
We also evaluate whether the gains of \algo{} can be explained only by using a learned safety detector to decide \emph{when} to intervene.
Inspired by conditional activation steering~\citep{lee2025programming}, we construct two conditional fixed-direction baselines on \texttt{Gemma-2-9B}.
Both baselines use the learned barrier as an unsafe-state detector during generation.
When the detector fires, \texttt{Conditional ActAdd} applies the same fixed activation-addition vector used by \texttt{ActAdd}, while \texttt{Conditional DirAbl} applies the same fixed directional-ablation operation used by \texttt{DirAbl}.
Thus, these baselines share the classifier-triggered intervention mechanism with \algo{}, but do not solve the CBF control objective and do not adapt the steering direction or magnitude to the current hidden state.

\begin{wraptable}[8]{r}{0.4\textwidth}
\vspace{-4mm}
\centering
\begin{minipage}{\linewidth}
\centering
\caption{Conditional fixed-direction steering baselines on \texttt{Gemma-2-9B}. We report HarmBench ASR; lower is better.}
\label{tab:conditional-fixed-steering}
\begin{small}
\setlength{\tabcolsep}{6pt}
\begin{tabular}{lc}
\toprule
\textbf{Method} & \textbf{ASR $\downarrow$} \\
\midrule
Base model & $19.50$ \\
Conditional ActAdd & $18.17_{\pm 0.00}$ \\
Conditional DirAbl & $12.50_{\pm 0.00}$ \\
\algo{} & $\mathbf{0.00_{\pm 0.00}}$ \\
\bottomrule
\end{tabular}
\end{small}
\end{minipage}
\vspace{-2mm}
\end{wraptable}

Table~\ref{tab:conditional-fixed-steering} shows that conditioning fixed-direction steering on a learned detector improves over applying no defense, but it does not match \algo{}.
This suggests that the benefit of \algo{} is not only deciding when to steer: the CBF update also computes a context-dependent correction direction, which more effectively moves unsafe hidden states toward the learned safe set.

\subsection{Compute Resources and Wall-Clock Cost}
\label{appendix:compute-resources}
Unless otherwise stated, the main experiments were run using two NVIDIA H200 GPUs with bfloat16 model weights.
The latency microbenchmarks in Appendix~\ref{sec:computational-benchmark-details} are the exception: they use a single H200 GPU with batch size $1$ to measure per-token steering overhead in the real-time decoding setting.

\para{Training.}
Hidden-state extraction is the main preprocessing cost before training the latent safety constraints.
For Qwen2-1.5B, hidden-state extraction typically takes about 15--25 minutes per seed; for Ministral-8B and Llama-2-7B, about 30--45 minutes; and for Gemma-2-9B, about 45--60 minutes.
Once hidden states are extracted, training \algo{} or \texttt{SaP} is comparatively lightweight: roughly 2--4 minutes per seed for Qwen2-1.5B, 3--6 minutes for Ministral-8B and Llama-2-7B, and 5--10 minutes for Gemma-2-9B.
The training-time baselines require more time because they update model-side intervention parameters: \texttt{Circuit Breaker} and \texttt{ReFT-r1} take less than one hour per model and seed under the hyperparameters in Appendix~\ref{sec:implementation-details-baselines}.
For the WildGuard experiment, each Qwen2-1.5B category barrier is trained independently; individual category jobs are short after hidden-state extraction and the 14 categories are parallelized across GPUs.

\para{Evaluation.}
For HarmBench static ASR evaluation, the approximate wall-clock time per model and seed is 6--20 minutes for Qwen2-1.5B, 6--15 minutes for Ministral-8B, 30--54 minutes for Llama-2-7B, and 30--40 minutes for Gemma-2-9B.
Adaptive-attack evaluation is more expensive and takes about 30--40 minutes per model and seed for the configured adaptive-evaluation run.
XSTest over-refusal evaluation uses 250 safe-only prompts and usually takes less than 10 minutes per model and seed, with the wall-clock time often dominated by batched API judge calls.
MMLU/GSM8K utility evaluation through the LM Evaluation Harness takes about 1--3 hours per model and method, depending on the model size and task subset.
The WildGuard composition evaluation on Qwen2-1.5B evaluates $1{,}120$ held-out WildGuardMix samples and takes about 10--30 minutes per seed after the category-specific barriers have been trained.

%% file: latex/additional_related_work.tex

\section{Additional Related Work on Inference-time Safety}
\label{appendix:additional-related-work}
This section situates \algo{} within inference-time safety methods, covering
text-level defenses, latent safety classifiers, activation steering,
adaptive activation steering, and constraint-based steering, including a
comparison with \texttt{SaP}, the most adjacent baseline to \algo{}.
These methods complement training-time
approaches~\citep{NeurIPS22_ouyang2022training,arXiv22_bai2022training,bai2022constitutionalaiharmlessnessai}
by adding an additional layer of defense during deployment,
since safety-tuned models can still be vulnerable to adversarial prompts~\citep{zou2023universal,liu2023autodan,shen2024anything,zeng2024johnny}.

\para{Text-level Defenses.}
\label{appendix:related-text-level-filters}
One line of inference-time safety work attempts to mitigate unsafe behavior
by inspecting input prompts or generated outputs.
Input filtering methods detect suspicious prompts before generation
~\citep{cao2024defending,jain2023baseline}, while input modification methods rewrite,
paraphrase, or otherwise transform prompts to weaken potential attacks~\citep{yi2025benchmarking,wei2026jailbreak,xie2023defending,robey2023smoothllm}.
Output filtering methods instead inspect generated text and block, revise, or
regenerate responses when safety violations are detected~\citep{phute2023llm,zhang2024parden,wang2024defending}.
These input and output filtering methods are attractive because they are model-agnostic and can be deployed
outside the target LLM, but they primarily rely on text-level patterns or heuristic filters.
As a result, they can be brittle to paraphrases, obfuscation, and
adaptive white-box attacks~\citep{wei2023jailbroken}.

Learned language-model classifiers provide stronger detection~\citep{mazeika2024harmbench,inan2023llamaguardllmbasedinputoutput} and can yield promising defense results when trained at
large scale~\citep{sharma2025constitutionalclassifiersdefendinguniversal}, but they typically require running an additional model over prompts
or responses. This additional checking introduces nontrivial computational
overhead, especially when the system must detect an unsafe output and then
regenerate a safer response.

\para{Latent Safety Classifiers.}
Recent work therefore studies classifiers over intermediate hidden states as a
more efficient alternative to text-space safety classification
~\citep{cunningham2026constitutional}. Hidden-state classifiers detect unsafe behavior from
model-internal representations rather than completed text, offering a promising
path toward lower-latency safety detection. However, detection-only systems still
primarily decide whether to block a response; they do not directly
modify generation toward a safer continuation.

\para{Activation Steering.} A promising direction for safer generation is activation steering,
which intervenes on model activations during generation.
Activation-addition methods modify hidden states along fixed directions, often using a linear addition~\citep{arditi2024refusal,rimsky2024steering,turner2023steering,zou2023representation,li2023inference,wu2025axbenchsteeringllmssimple}.
Recent work further advances this direction with affine and nonlinear steering~\citep{singh2024representation,vu2026angular,you2026sphericalsteeringgeometryawareactivation}.
These interventions are simple but can be coarse,
as the same direction is applied across contexts and may degrade utility.

\para{Adaptive Activation Steering.}
Beyond fixed-vector activation steering, recent work has used learned models
to provide more flexible inference-time steering signals.
A simple form is conditional steering~\citep{lee2025programming},
where a fixed steering vector is applied only when a latent classifier detects
undesired behavior.
Other methods predict context-dependent steering directions,
for example using hypernetwork-style modules~\citep{sun2025hypersteeractivationsteeringscale},
or directly transform hidden states to modify model behavior~\citep{wang2025truthflowtruthfulllmgeneration}.

A related line of work learns scalar signals over latent representations
and uses them to guide iterative steering.
For example, \texttt{PPLM} uses gradients from attribute models during decoding~\citep{dathathri2020plugplaylanguagemodels},
\texttt{RE-Control} learns value functions over hidden states for representation editing~\citep{kong2024aligning},
\texttt{ODESteer} learns a density-ratio function and follows its gradient
through ODE integration~\citep{zhao2026odesteer},
and \texttt{BRT-Align} uses learned safety values with perturbation search
to steer away from unsafe regions~\citep{karnik2025preemptivedetectionsteeringllm}.
However, these methods use learned signals primarily to define a direction or
search objective for iterative steering, rather than as constraints that must be
satisfied during generation.

In contrast, \algo{} treats learned neural safety classifiers as explicit constraints on latent representations.
Rather than optimizing activations toward a desirable direction, it uses the classifier as a Control Barrier Function and derives
constraint-guided steering updates for safety-critical generation.
This perspective is especially suitable for LLM safety defense, where the goal is
not only to encourage preferred behavior, but to prevent trajectories from
entering unsafe regions.

\para{Constraint-based Steering.}
Recent work such as \texttt{SaP}~\citep{ICML2025_chenlearning} suggests a natural combination:
using learned latent safety classifiers not only as detection signals or steering objectives,
but also as explicit constraints on hidden representations.
Under this view, the classifier decision boundary defines a safe region in latent space,
and inference-time steering aims to keep hidden states within, or move them toward,
that region during generation.

\texttt{SaP} is the closest prior method to \algo{}.
Like \algo{}, it learns safety constraints from hidden states and safety labels,
then modifies hidden states at inference time to remain within the learned safe region.
We therefore include \texttt{SaP} as a main baseline.
Both \texttt{SaP} and \algo{} use hidden states and safety labels to learn latent safety constraints,
but they differ in the constraint class and the inference-time enforcement mechanism.
\texttt{SaP} learns linear polytope constraints and enforces them through
iterative gradient-descent at inference time.
In contrast, \algo{} learns nonlinear CBF constraints and derives closed-form
local steering updates, enabling more expressive safety boundaries with
substantially lower per-step steering overhead.

%% file: latex/faq.tex

\section{Additional Discussion and Clarifications}
\label{asec:additional-discussion}

\subsection{Assumptions for Theoretical Safety}
\label{asec:clarifications-guarantees}

\paragraph{Q1. Does \algo{} provide a worst-case safety guarantee?}\mbox{}\\
\noindent\textbf{Answer.}
No.
In the main paper, we clarify that although Control Barrier Functions are a standard tool for maintaining safety within a safe set, the LLM setting does not admit a perfect worst-case guarantee.
Such a guarantee would require a perfect classifier or barrier whose safe set exactly corresponds to text-level safety for all possible prompts and generations.
This is not realistic for two reasons.
First, safety data are finite, while the space of possible LLM generations and hidden states is extremely large, so a learned barrier cannot certify every possible trajectory.
Second, there is a latent--semantic safety gap: even if a hidden state appears safe under the learned classifier, later layers and decoding can still produce unsafe text, as discussed in Appendix~\ref{asec:limitations}.
Our goal is therefore not to prove worst-case text-level safety, but to steer responses toward an approximated latent safe set learned from data.
This approximation is still practically meaningful: recent work shows that hidden-state classifiers can be useful in production-grade safety systems for frontier models~\citep{cunningham2026constitutional}, suggesting that learned latent safe sets can be strong enough to support effective deployment-time interventions.

\paragraph{Q2. Is the updated hidden state guaranteed to be inside the safe set after one steering step?}\mbox{}\\
\noindent\textbf{Answer.}
No.
Our implementation applies a one-step CBF update during decoding.
The local CBF condition moves the hidden state toward the learned safe set, and the theorem characterizes exponential progress toward the safe set under the stated assumptions, but this does not mean that the updated hidden state $h'$ must lie inside the safe set after a single finite update.
In practice, however, the one-step update is already effective: our experiments show strong reductions in adversarial attack success while preserving utility across model families.

\subsection{Data, Supervision, and Generalization}
\label{asec:clarifications-data}

\paragraph{Q3. Does \algo{} generalize to unseen harmful behaviors?}\mbox{}\\
\noindent\textbf{Answer.}
Like other data-driven safety methods, \algo{} depends on the coverage and quality of the safety data used to train the barrier. We do not expect a barrier trained for one safety concept to perfectly generalize to all unrelated out-of-distribution harms. Instead, the framework is modular: new barriers can be trained for newly identified safety concepts and composed with existing barriers. That said, we consistently observe empirical generalization beyond the exact training examples. In HarmBench, we train on the training behaviors and evaluate on held-out adversarial behaviors, where \algo{} still substantially reduces attack success. In WildGuard, even individual category barriers slightly reduce harmful generations outside their exact category, while composing category-specific barriers substantially improves coverage. Finally, in the adaptive-attack evaluation, \algo{} also reduces attack success even though the adaptive attack trajectories are not part of the training data. 
Overall, despite being trained on the same data as the learned baselines, \algo{} reduces harmful generations more effectively.

\paragraph{Q4. What supervision is used to train \algo{}?}\mbox{}\\
\noindent\textbf{Answer.}
\algo{} uses binary safe/unsafe supervision over hidden states.
Sequence-level safety labels are inherited by token-level hidden states extracted from the same response, producing hidden-state examples labeled as safe or unsafe.
We do not require preference pairs, ranked responses, or paired chosen/rejected examples.
Appendix~\ref{appendix:hidden-state-counts} gives the token-level hidden-state training details.

\paragraph{Q5. How is \algo{} different from preference-tuned safety methods?}\mbox{}\\
\noindent\textbf{Answer.}
Training-time alignment methods such as RLHF~\citep{NeurIPS22_ouyang2022training,arXiv22_bai2022training,bai2022constitutionalaiharmlessnessai}, Safe-RLHF~\citep{dai2023safe}, and DPO-style methods~\citep{arXiv23_rafailov2023direct,zhou2024beyond,liu2024enhancing} update model parameters using preference, reward, or safety-constraint objectives.
In contrast, \algo{} trains a separate latent barrier from binary safety labels and applies it as an inference-time controller on top of an already aligned model.
This makes \algo{} complementary to preference tuning: the base model can already be RLHF/DPO-tuned, while the barrier can be updated or composed for new safety risks without retraining the LLM.

\subsection{BarrierSteer Variants and Steering Geometry}
\label{asec:clarifications-variants-geometry}

\paragraph{Q6. Which \algo{} variant is preferred?}\mbox{}\\
\noindent\textbf{Answer.}
We recommend \algo{} (LSE) as the default variant. Top-2 selects the two most violated barriers and is useful as a simple, fast ablation for isolating steering-strength effects, but when many safety categories are composed, ignoring the remaining constraints can be suboptimal. QP is useful as a more exact optimization-based reference. LSE is the best default for deployment and main reporting because it accounts for all active constraints, matches the strong defense behavior of the multi-constraint formulations, and has the lowest measured steering latency among our variants in Table~\ref{tab:latency}. Appendix~\ref{appendix:variant-selection-details} provides steering-strength ablations, and Appendix~\ref{sec:computational-benchmark-details} provides the detailed latency comparison among variants.

\paragraph{Q7. Can \algo{} be adapted for steering geometries beyond linear addition?}\mbox{}\\
\noindent\textbf{Answer.}
As discussed in Appendix~\ref{appendix:additional-related-work}, recent activation-steering 
work has moved beyond fixed linear addition toward affine, angular, spherical, and 
other nonlinear or norm-preserving geometries~\citep{singh2024representation,vu2026angular,you2026sphericalsteeringgeometryawareactivation}. 
\algo{} is complementary to these methods: the current implementation 
applies the CBF correction as an additive hidden-state update for 
efficiency, but the learned barrier can also provide a 
local safety signal for other update geometries. 
Extending BarrierSteer's classifier-guided CBF constraints 
to norm-preserving or rotation-based updates is a natural direction for future work.

\subsection{Evaluation, Robustness, and Deployment Scope}
\label{asec:clarifications-evaluation-deployment}

\paragraph{Q8. How robust is \algo{} to adaptive attacks?}\mbox{}\\
\noindent\textbf{Answer.}
Adaptive attacks are stronger than static HarmBench attacks because they optimize against the defended model while the steering mechanism is active.
Table~\ref{tab:joint_results} reports adaptive-attack ASR together with static ASR, over-refusal, and utility across four model families.
\algo{} consistently improves over the undefended model and remains competitive with or stronger than other safety baselines.
The adaptive results also show that near-zero static ASR should not be interpreted as absolute robustness: stronger attacks can still partially bypass the defense, but \algo{} substantially improves robustness under an adaptive threat model.
Additional steering-strength sweeps in Appendix~\ref{appendix:variant-selection-details} show how the adaptive robustness--utility trade-off changes with $\alpha$.

\paragraph{Q9. Does \algo{} over-refuse benign prompts?}\mbox{}\\
\noindent\textbf{Answer.}
We evaluate over-refusal using XSTest safe-only prompts and report it alongside ASR and utility in Table~\ref{tab:joint_results}.
The results show that \algo{} can increase over-refusal relative to the base model, but it avoids the most severe over-refusal observed for some stronger baselines and can be tuned through the steering strength $\alpha$.
This safety--utility trade-off is expected for inference-time steering: stronger interventions improve attack resistance but may reject more benign prompts.
Appendix~\ref{appendix:variant-selection-details} provides the corresponding $\alpha$ sweeps, and Appendix~\ref{appendix:extended-utility-results} reports additional utility benchmarks beyond MMLU and GSM8K.

\paragraph{Q10. Is \algo{} production-grade?}\mbox{}\\
\noindent\textbf{Answer.}
Not yet.
Our experiments use medium-scale open-weight LLMs, so we do not claim production-grade deployment on frontier systems.
However, the broader direction is motivated by recent evidence that hidden-state classifiers can be production-grade safety detectors for frontier models~\citep{cunningham2026constitutional}.
\algo{} builds on this direction by using learned hidden-state safety signals not only to detect unsafe trajectories, but also to steer generation with low additional latency; Table~\ref{tab:latency} and Appendix~\ref{sec:computational-benchmark-details} quantify this overhead.
Demonstrating production-grade steering would require direct frontier-model hidden-state access and deployment-scale evaluation, which we leave to future work.

%% file: latex/limitations.tex

\section{Limitations and Future Work}
\label{asec:limitations}

\para{Dependence on supervised safety labels.}
\algo{} learns barriers from labeled safe and unsafe hidden states, so its effectiveness depends on the coverage and quality of the safety data.
If a harmful behavior is absent from the training dataset, the corresponding latent region may not be captured by the learned CBFs.
Future work should study broader safety taxonomies, active data collection for uncovered harms, and continual updates that add new barriers for emerging risk categories.

\para{Latent safety is not identical to semantic safety.}
The control-theoretic guarantees in this paper apply to learned latent constraints, not directly to text-level safety, and are conditional on the learned barriers accurately representing the intended safety concepts.
One way to mitigate this is to train stronger latent constraints with larger-scale red-teaming data~\citep{cunningham2026constitutional}.

\para{Evaluation scope.}
Our experiments cover multiple open-weight instruction-tuned models, HarmBench adversarial attacks, XSTest over-refusal, utility benchmarks, adaptive attacks, and WildGuard modular composition.
However, they do not establish robustness against all possible adaptive attackers or all deployment settings.
Future work should evaluate larger frontier models when hidden-state access is available, broader multilingual and multi-turn safety settings, and stronger adaptive attacks that optimize jointly over prompts and steering responses.

\para{Geometry and intervention form.}
The current implementation applies the CBF correction as an additive hidden-state update.
Although the update direction is computed from nonlinear barrier geometry, additive interventions can still alter activation norms and may not be optimal for all models.
Future work can combine BarrierSteer's classifier-guided CBF constraints with norm-preserving steering updates.

\para{Safety--utility trade-off.}
\algo{} improves the safety--utility trade-off relative to the baselines we evaluate, but it does not remove the trade-off.
Increasing steering strength can reduce unsafe generations while increasing over-refusal or degrading downstream utility.
Future work should develop adaptive steering schedules that tune intervention strength based on uncertainty, task context, and deployment-specific risk tolerance.

%% file: latex/broader_impact.tex

\section{Broader Impact}
\label{asec:broader-impact}

This work aims to improve the safety of large language models by steering model activations away from unsafe generations during inference.
Its potential positive impact is to provide an additional defense layer that can complement training-time alignment, reduce harmful outputs, and make safety interventions more modular and efficient.
Because the method does not require modifying the base model parameters, it can complement existing hidden-state classifiers, which have been shown to be useful for frontier models at deployment.

The work also has potential risks.
First, safety-steering systems can create a false sense of security if users interpret lower attack success rates as robustness.
Second, the same evaluation infrastructure used to measure jailbreak robustness can be misused to search for stronger attacks.
Third, learned barriers may encode biases or blind spots from the training data and classifiers, leading to increased over-refusal or uneven effects across benign users and topics.
Fourth, the same constraint-learning and steering mechanism could be misused by training barriers on undesirable behaviors and steering the model toward, rather than away from, unsafe responses.
We mitigate these risks by evaluating adaptive attacks, reporting over-refusal and utility, discussing limitations, and framing the method as a complementary safety layer rather than a complete guarantee of safe behavior.

%% file: latex/llm_usage.tex

\section{Declaration of LLM Usage}
\label{asec:llm-usage}

LLMs are the main objects of study in this work.
Our method extracts and steers hidden states from instruction-tuned LLMs, and our experiments evaluate their generated responses using safety classifiers and benchmark tasks.
These methodological uses are described throughout Secs.~\ref{sec:problem}--\ref{sec:experiments}.
We also used large language models to support implementation and paper editing.
Specifically, LLM assistance was used for code-writing support, LaTeX editing, and wording refinement.
All technical claims, experimental results, citations, and final manuscript content were verified by the authors.\\